\documentclass[10pt,twocolumn,letterpaper]{article}
\usepackage[pagenumbers]{cvpr}

\usepackage{graphicx}
\usepackage{amsmath}
\usepackage{amssymb}
\usepackage{booktabs}

\usepackage[pagebackref,breaklinks,colorlinks]{hyperref}

\usepackage[capitalize]{cleveref}
\crefname{section}{Sec.}{Secs.}
\Crefname{section}{Section}{Sections}
\Crefname{table}{Table}{Tables}
\crefname{table}{Tab.}{Tabs.}

\usepackage[accsupp]{axessibility} 

\usepackage{amsfonts,bm}
\usepackage{physics}
\usepackage{wrapfig}
\usepackage{multirow}
\usepackage{caption}
\usepackage{color,soul}
\usepackage{xcolor}
\usepackage[frozencache,cachedir=minted-cache]{minted}

\usepackage{mathtools}
\usepackage{amsthm}

\theoremstyle{plain}
\newtheorem{theorem}{Theorem}

\newtheorem{lemma}{Lemma}

\newtheorem{assumption}{Assumption}
\theoremstyle{definition}
\newtheorem{definition}{Definition}

\theoremstyle{remark}

\begin{document}

\title{VNE: An Effective Method for Improving Deep Representation \\ by Manipulating Eigenvalue Distribution}

\author{
Jaeill Kim\textsuperscript{\rm 1},
Suhyun Kang\textsuperscript{\rm 1},
Duhun Hwang\textsuperscript{\rm 1},
Jungwook Shin\textsuperscript{\rm 1},
Wonjong Rhee\textsuperscript{\rm 1,2,3}\thanks{Corresponding author}\\
\textsuperscript{\rm 1} Department of Intelligence and Information,
Seoul National University\\
\textsuperscript{\rm 2} Interdisciplinary Program in Artificial Intelligence (IPAI),
Seoul National University\\
\textsuperscript{\rm 3} Research Institute for Convergence Science,
Seoul National University
\\
{\tt\small \{jaeill0704, su\_hyun, yelobean, jungwook.shin, wrhee\}@snu.ac.kr}
}

\maketitle

\begin{abstract}
Since the introduction of deep learning, a wide scope of representation properties, such as decorrelation, whitening, disentanglement, rank, isotropy, and mutual information, have been studied to improve the quality of representation. However, manipulating such properties can be challenging in terms of implementational effectiveness and general applicability. To address these limitations, we propose to regularize von Neumann entropy~(VNE) of representation. First, we demonstrate that the mathematical formulation of VNE is superior in effectively manipulating the eigenvalues of the representation autocorrelation matrix. Then, we demonstrate that it is widely applicable in improving state-of-the-art algorithms or popular benchmark algorithms by investigating domain-generalization, meta-learning, self-supervised learning, and generative models. In addition, we formally establish theoretical connections with rank, disentanglement, and isotropy of representation. Finally, we provide discussions on the dimension control of VNE and the relationship with Shannon entropy. Code is available at: \url{https://github.com/jaeill/CVPR23-VNE}.

\end{abstract}

\section{Introduction}
\label{sec:introduction}

\begin{figure}[t!]
  \centering
  \subfloat[Domain generalization]{
    \includegraphics[width=0.43\columnwidth]{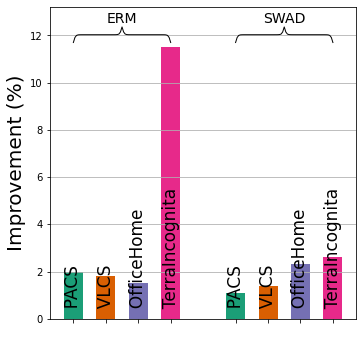}
    \vspace{-0.25cm}
  }
  \subfloat[Meta-learning]{
    \includegraphics[width=0.43\columnwidth]{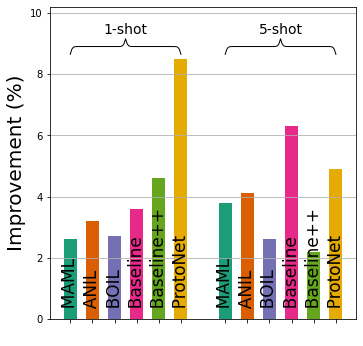}
    \vspace{-0.25cm}
  }
  \textbf{}
  \subfloat[Self-supervised learning]{
    \includegraphics[width=0.43\columnwidth]{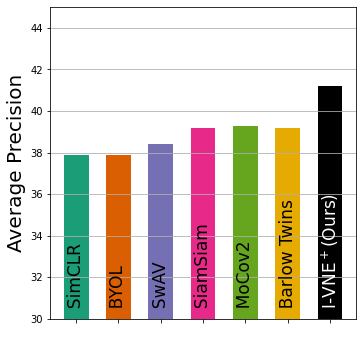}
    \vspace{-0.25cm}
  }
  \subfloat[GAN]{
    \includegraphics[width=0.43\columnwidth]{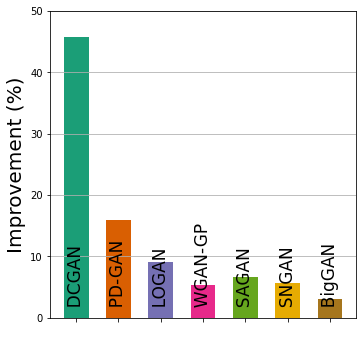}
    \vspace{-0.25cm}
  }
\vspace{-0.2cm}
\caption{\label{figure1}
General applicability of VNE: performance of state-of-the-art algorithms or popular benchmark algorithms can be further improved by regularizing von Neumann entropy (full result tables will be provided in Section~\ref{sec:experiments}). 
(a) Domain generalization: relative improvements over ERM and SWAD~(current state-of-the-art). 
(b) Meta-learning: relative improvements over six popular benchmark algorithms. 
(c) Self-supervised learning: performance comparison against the current state-of-the-art algorithms for COCO detection.
(d) GAN: relative improvements in Fréchet Inception Distance~(FID) for seven popular benchmark algorithms.
}
\vspace{-0.4cm}
\end{figure}

Improving the quality of deep representation by pursuing a variety of properties in the representation has been adopted as a conventional practice.
To learn representations with useful properties,
various methods have been proposed to manipulate the representations.
For example, \textit{decorrelation} reduces overfitting, enhances generalization in supervised learning~\cite{cogswell2015reducing,xiong2016regularizing}, and helps in clustering~\cite{tao2021clustering}.
\textit{Whitening} improves convergence and generalization in supervised learning~\cite{huang2018decorrelated,huang2019iterative,desjardins2015natural,luo2017learning}, improves GAN stability~\cite{siarohin2018whitening}, and helps in domain adaptation~\cite{roy2019unsupervised}.
\textit{Disentanglement} was proposed as a desirable property of representations~\cite{bengio2009learning,achille2018emergence,hjelm2018learning}.
Increasing \textit{rank} of representations was proposed to resolve the dimensional collapse phenomenon in self-supervised learning~\cite{hua2021feature,jing2021understanding}.
\textit{Isotropy} was proposed to improve the downstream task performance of BERT-based models in NLP tasks~\cite{li2020sentence,su2021whitening}.
\textit{Preventing informational collapse} (also known as representation collapse) was proposed as a successful learning objective in non-contrastive learning~\cite{bardes2021vicreg, zbontar2021barlow}.
In addition, \textit{maximizing mutual information} was proposed as a successful learning objective in contrastive learning~\cite{henaff2020data, oord2018representation, tian2020contrastive}.

Although aforementioned properties are considered as desirable for useful representations,
typical implementational limitations, such as dependency to specific architectures or difficulty in proper loss formulation, inhibited the properties from being more popularly adopted.
For example, the methods for whitening~\cite{huang2018decorrelated,huang2019iterative,desjardins2015natural,luo2017learning,siarohin2018whitening,roy2019unsupervised}, isotropy~\cite{li2020sentence,su2021whitening}, and rank~\cite{hua2021feature} are typically dependent on specific architectures (e.g., decorrelated batch normalization~\cite{huang2018decorrelated} and normalizing flow~\cite{li2020sentence}).
Regarding disentanglement and mutual information, loss formulations are not straightforward because measuring disentanglement generally relies on external models~\cite{higgins2017beta,eastwood2018framework,kim2018disentangling,chen2018isolating,glorot2011domain} or is tuned for a specific dataset~\cite{karras2019style} and 
formulating mutual information in high-dimensional spaces is notoriously difficult~\cite{tschannen2019mutual} and only tractable lower bound can be implemented by training additional critic functions~\cite{poole2019variational}.
Meanwhile, several decorrelation methods~\cite{cogswell2015reducing,xiong2016regularizing,bardes2021vicreg, zbontar2021barlow} have implemented model-agnostic and straightforward loss formulations that minimize the Frobenius norm between the autocorrelation matrix $\mathcal{C}_{\text{auto}}$ (or crosscorrelation matrix $\mathcal{C}_{\text{cross}}$) and a scale identity matrix $c \cdot I$ for an appropriate $c>0$.
Because of the easiness of enforcing decorrelation via a simple loss formulation, these decorrelation methods can be considered to be generally applicable to a wide scope of applications. However, the current implementation of the loss as a Frobenius norm can exhibit undesirable behaviors during learning and thus fail to have a positive influence as we will explain further in Section~\ref{sec:vne_vs_Frobenius}.

To address the implementational limitations,
this study considers the eigenvalue distribution of the autocorrelation matrix $\mathcal{C}_{\text{auto}}$.
Because $\mathcal{C}_{\text{auto}}$ converges to scalar identity matrix $c \cdot I$ for an appropriate $c>0$ if and only if the eigenvalue distribution of $\mathcal{C}_{\text{auto}}$ converges to a uniform distribution,
it is possible to control the eigenvalue distribution using methods that are different from Frobenius norm.
To this end, we adopt a mathematical formulation from quantum information theory and introduce \textit{von Neumann entropy}~(VNE) of deep representation, a novel method that directly controls the eigenvalue distribution of $\mathcal{C}_{\text{auto}}$ via an entropy function.
Because entropy function is an effective measure for the uniformity of underlying distribution and can handle extreme values, optimizing the entropy function is quite stable and does not possess implementational limitations of previous methods.

In addition to the effectiveness of VNE on manipulating the eigenvalue distribution of $\mathcal{C}_{\text{auto}}$,
we demonstrate that regularizing VNE is widely beneficial in improving the existing studies. As summarized in Figure~\ref{figure1}, performance improvement is significant and consistent.
Moreover, theoretical connections between VNE and the popular representation properties are formally proven and support the empirical superiority.
Thanks to the implementational effectiveness and theoretical connections, VNE regularizer can effectively control not only von Neumann entropy but also other theoretically related properties, including rank and isotropy.
Our contributions can be summarized as below:
\begin{itemize}
\vspace{-0.1cm}
    \item We introduce a novel representation regularization method, von Neumann entropy of deep representation.
    \item We describe VNE's \textit{implementational effectiveness} (in Section~\ref{sec:vonneumannentropy}).
    \item We demonstrate \textit{general applicability} of VNE by improving current state-of-the-art methods in various tasks and achieving a new state-of-the-art performance in self-supervised learning and domain generalization (in Section~\ref{sec:experiments}).
    \item We provide \textit{theoretical connections} by proving that VNE is theoretically connected to rank, disentanglement, and isotropy of representation (in Section~\ref{sec:propertiesofvne}).
\vspace{-0.1cm}
\end{itemize}

\section{Implementational Effectiveness of VNE}
\label{sec:vonneumannentropy}

Even though the von Neumann entropy originates from quantum information theory, we focus on its mathematical formulation to understand why it is effective for manipulating representation properties. We start by defining the autocorrelation matrix.

\subsection{Autocorrelation of Representation}
For a given mini-batch of $N$ samples, the representation matrix can be denoted as $\bm{H} = [\bm{h}_1, \bm{h}_2, ..., \bm{h}_N]^T \in \mathbb{R}^{N \times d}$, where $d$ is the size of the representation vector. For simplicity, we assume $L_2$-normalized representation vectors satisfying $||\bm{h}_i||_2=1$ as in~\cite{wang2020understanding,parkhi2015deep,schroff2015facenet,liu2017sphereface,wang2017normface,mettes2019hyperspherical,xu2018spherical}. Then, the autocorrelation matrix of the representation is defined as:
\vspace{-0.3cm}
\begin{equation}
\label{def:density_rho}
\mathcal{C}_{\text{auto}} \triangleq \sum_{i=1}^N\frac{1}{N}\bm{h}_i\bm{h}_i^T=\bm{H}^T\bm{H}/N.
\vspace{-0.3cm}
\end{equation}
For $\mathcal{C}_{\text{auto}}$'s eigenvalues $\{\lambda_{j}\}$, it can be easily verified that $\sum_j \lambda_{j}=1$ and $\forall_j~\lambda_{j} \ge 0$ because $||\bm{h}_i||_2=1$ and $\mathcal{C}_{\text{auto}} \ge 0$. For the readers familiar with quantum information theory, $\mathcal{C}_{\text{auto}}$ is used in place of the density matrix $\rho$ of Supplementary~\ref{sec:quantum_prelim}~(a brief introduction to quantum theory).

$\mathcal{C}_{\text{auto}}$ is closely related to a variety of representation properties. In the extreme case of $\mathcal{C}_{\text{auto}} \rightarrow c \cdot I_d$, where $c$ is an adequate positive constant, the eigenvalue distribution of $\mathcal{C}_{\text{auto}}$ becomes perfectly uniform. Then, the representation $\bm{h}$ becomes decorrelated~\cite{cogswell2015reducing}, whitened~\cite{huang2018decorrelated}, full rank~\cite{hua2021feature}, and isotropic~\cite{vershynin2018high}. In the case of self-supervised learning, it means prevention of informational collapse\cite{bardes2021vicreg, zbontar2021barlow}. 

Besides its relevance to numerous representation properties, regularizing $\mathcal{C}_{\text{auto}}$ is of a great interest because it permits a simple implementation. Unlike many of the existing implementations that can be dependent on specific architecture or dataset, difficult to implement as a  loss, or dependent on successful learning of external models, $\mathcal{C}_{\text{auto}}$ can be regularized as a simple penalty loss. 
Because $\mathcal{C}_{\text{auto}}$ is closely related to a variety of representation properties and because it permits a broad applicability, we focus on $\mathcal{C}_{\text{auto}}$ in this study.

\subsection{Regularization with Frobenius Norm}
A popular method for regularizing the eigenvalues of $\mathcal{C}_{\text{auto}}$ is to implement the loss of Frobenius norm as shown below.
\vspace{-0.3cm}
\begin{equation}
\begin{split}
\mathcal{L}_{\text{Frobenius}}&\triangleq|| \mathcal{C}_{\text{auto}} - c \cdot I_d ||_F^2\\
&= \sum_{i} (\mathcal{C}_{i,i}-c)^2  + \sum_{i} \sum_{j \ne i} \mathcal{C}_{i,j}^2 \label{eq:Frobenius}
\end{split}
\vspace{-0.5cm}
\end{equation}
$\mathcal{C}_{i,j}$ is the $(i,j)$ element of $\mathcal{C}_{\text{auto}}$ and $c$ is an adequate positive constant. While this approach has been widely adopted in the previous studies including DeCov~\cite{cogswell2015reducing}, cw-CR~\cite{choi2019utilizing}, SDC~\cite{xiong2016regularizing}, Barlow Twins~\cite{zbontar2021barlow}, and VICReg~\cite{bardes2021vicreg}, it can be ineffective for controlling eigenvalues as we will show in Section~\ref{sec:vne_vs_Frobenius}.

\subsection{Regularization with Von Neumann Entropy}
\label{sec:vne_definition}
Von Neumann entropy of autocorrelation is defined as the Shannon entropy over the eigenvalues of $\mathcal{C}_{\text{auto}}$. The mathematical formulation is shown below. 
\vspace{-0.1cm}
\begin{equation}
\label{eq:vne_formula}
S(\mathcal{C}_{\text{auto}}) \triangleq -\sum_j \lambda_{j} \log{\lambda_{j}}.
\vspace{-0.2cm}
\end{equation}
As shown in Lemma~\ref{lem:entropy} of Supplementary~\ref{sec:prop_proof}, $S(\mathcal{C}_{\text{auto}})$ ranges between zero and $\log{d}$.
Implementing of VNE regularization is simple. When training an arbitrary task $\mathcal{T}$, we can subtract $\alpha \cdot S(\mathcal{C}_{\text{auto}})$ from the main loss $\mathcal{L}_{\mathcal{T}}$.
\vspace{-0.1cm}
\begin{equation}
\mathcal{L}_{\mathcal{T}\text{+VNE}}=\mathcal{L}_{\mathcal{T}} - \alpha \cdot S(\mathcal{C}_{\text{auto}}).
\vspace{-0.1cm}
\end{equation}
Note that training $\mathcal{T}$ with $\mathcal{L}_{\mathcal{T}\text{+VNE}}$ is denoted as \textit{VNE$^+$} if $\alpha>0$, \textit{VNE$^-$} if $\alpha<0$, and \textit{Vanilla} if $\alpha=0$.
The PyTorch implementation of $S(\mathcal{C}_{\text{auto}})$ can be found in Figure~\ref{figure:vne_code} of Supplementary~\ref{sec:main_algo}.
Computational overhead of VNE calculation is light, as demonstrated in Table~\ref{tab:computational_overhead} of Supplementary~\ref{sec:computational_overhead}.

\subsection{Frobenius Norm vs. Von Neumann Entropy}
\label{sec:vne_vs_Frobenius}

The formulation of von Neumann entropy in Eq.~(\ref{eq:vne_formula}) exhibits two distinct differences when compared to the formulation of Frobenius norm in Eq.~(\ref{eq:Frobenius}). First, while Frobenius norm deals with all the elements of $\mathcal{C}_{\text{auto}} \in \mathbb{R}^{d \times d}$, VNE relies on an eigenvalue decomposition to identify the eigenvalues of the current model under training and focuses on the current $d$ eigenvalues only. Second, while Frobenius norm can manifest an undesired behavior when some of the eigenvalues are zero and cannot be regulated toward $c$, VNE gracefully handles such dimensions because $0\cdot \log{0} = 0$.

To demonstrate our points, we have performed a supervised learning with ResNet-18 and three datasets. The results are shown in Table~\ref{tab:variance_collapse} where regularization with Frobenius norm causes many neurons to become dead. Instead of focusing on the eigenvalues, Frobenius norm takes a shortcut of making many of the $d=512$ dimensions unusable and fails to recover. Note that \textit{VNE$^+$} and  \textit{VNE$^-$} do not present such a degenerate behavior. 
We have repeated the supervised experiment with ResNet-18, but this time using a relatively sophisticated dataset of ImageNet-1K. The distribution of eigenvalues are shown in Figure~\ref{figure:vne_fro}(a) where Frobenius norm fails to affect the distribution. \textit{VNE$^+$} and \textit{VNE$^-$}, however, successfully make the distribution more uniform and less uniform, respectively. Finally, the learning history of Frobenius norm loss for a self-supervised learning is shown in Figure~\ref{figure:vne_fro}(b). While Barlow Twins~\cite{zbontar2021barlow} is a well-known method, the Frobenius norm loss can be better manipulated by regularizing \textit{VNE$^+$} instead of regularizing the Frobenius norm itself.

\begin{table}[t!]
\centering
\resizebox{0.6\columnwidth}{!}{
\begin{tabular}{@{}lccc@{}}
\toprule
Method                           & \multicolumn{3}{c}{Dead units} \\ \cmidrule(l){2-4} 
                                 & CIFAR-10  & STL-10 & CIFAR-100 \\ \midrule
Vanilla                          & 0         & 0      & 0         \\
VNE$^-$                          & 0         & 0      & 1         \\
VNE$^+$                          & 2         & 0      & 1         \\
$\mathcal{L}_{\text{Frobenius}}$ & 447       & 365    & 325       \\ \bottomrule
\end{tabular}
}
\vspace{-0.2cm}
\caption{\label{tab:variance_collapse}
Count of dead units (dead neurons) when training ResNet-18 with the standard cross-entropy loss. The penultimate layer's representation with $d=512$ was analyzed. 
}
\end{table}

\begin{figure}[t!]
\vspace{-0.25cm}
  \centering
  \subfloat[\label{figure:vne_fro_a}Eigenvalue distribution (ordered)]{
    \includegraphics[width=0.47\columnwidth]{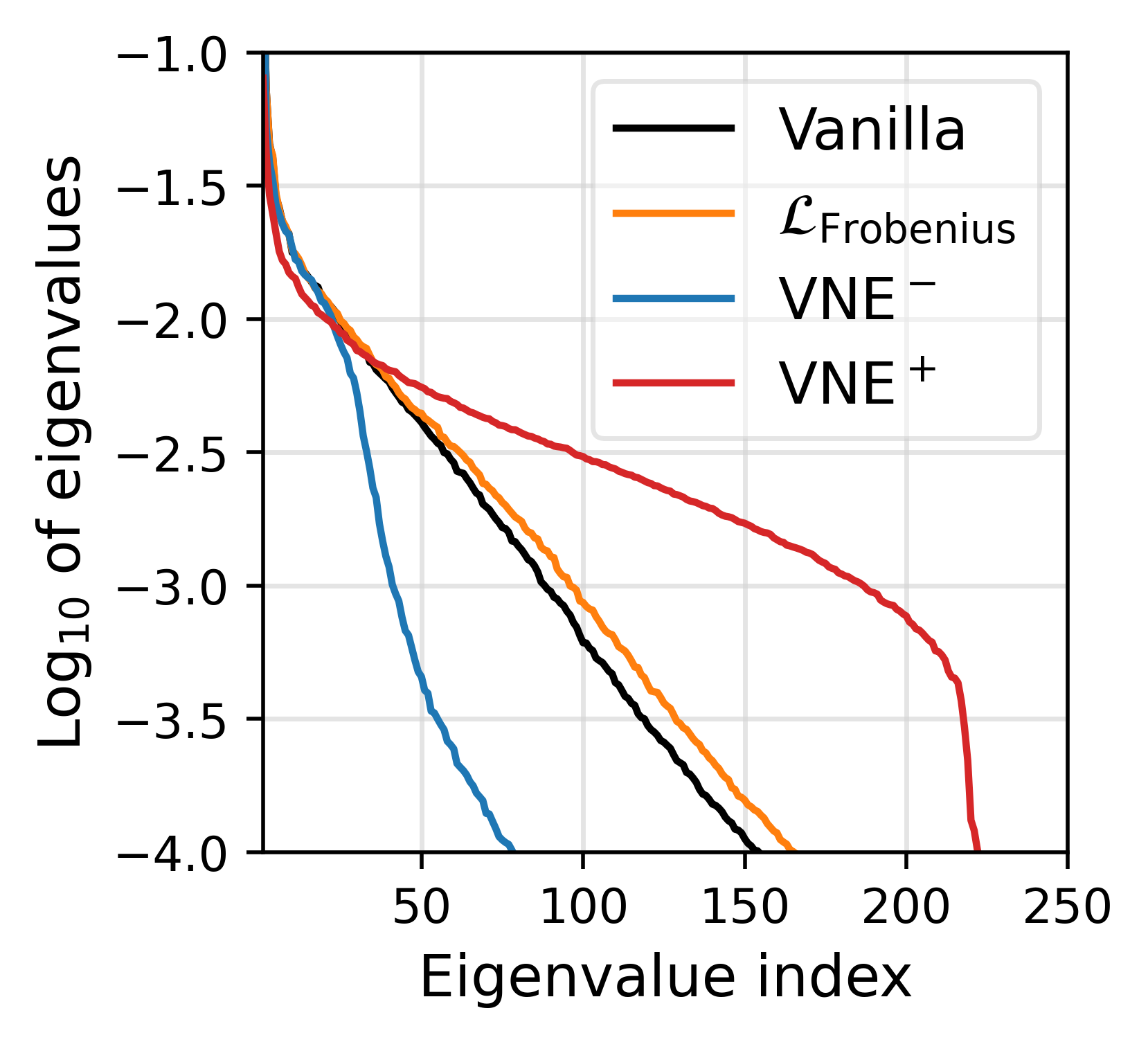}
  }
  \subfloat[\label{figure:vne_fro_b}$|| \mathcal{C}_{\text{auto}} - c \cdot I_d ||_F^2$]{
    \includegraphics[width=0.45\columnwidth]{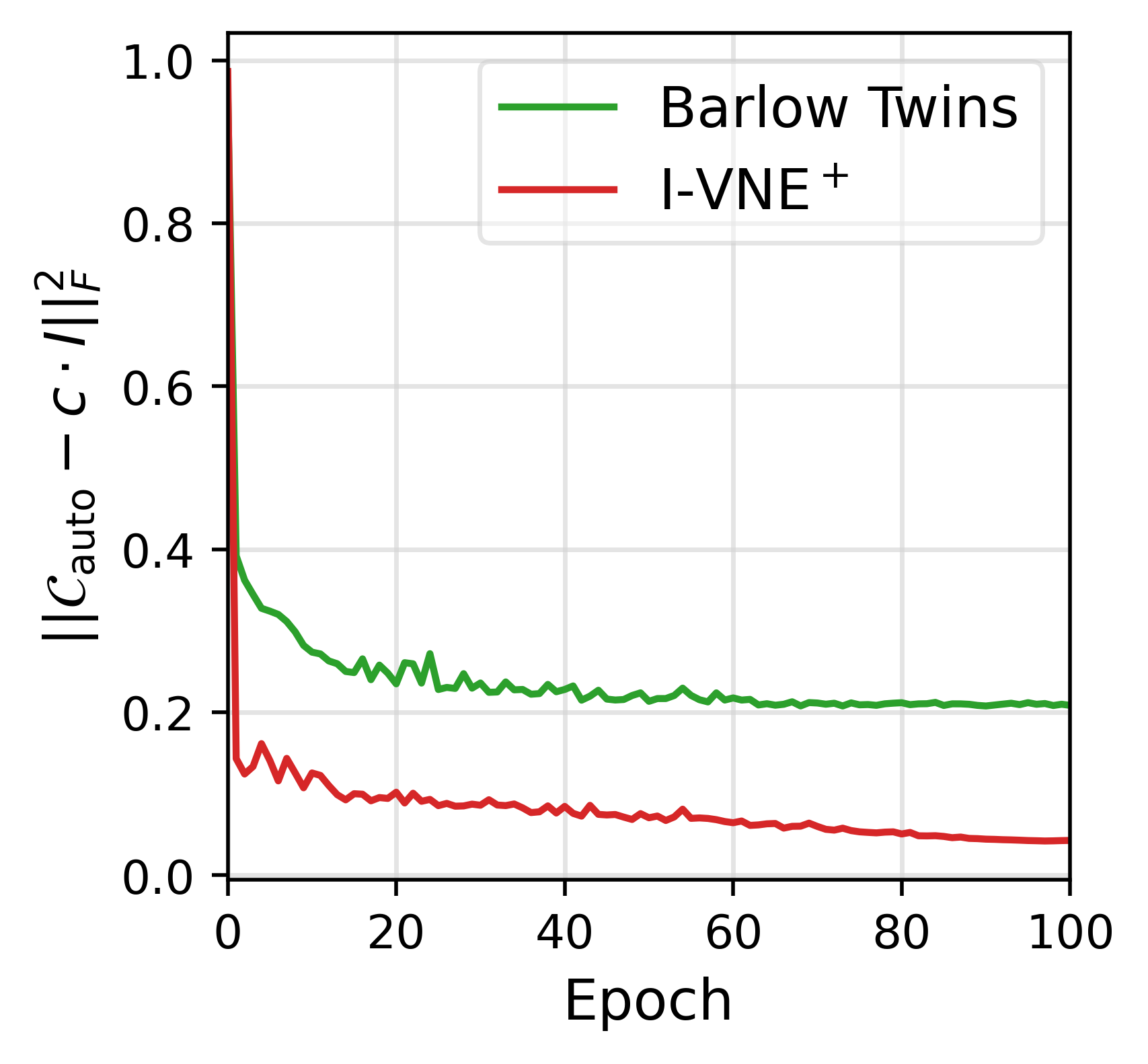}
  }
\vspace{-0.2cm}
\caption{\label{figure:vne_fro} (a) Eigenvalue distribution when training ResNet-18 with the standard cross-entropy loss (dataset: ImageNet-1K, $d=512$). (b) Frobenius norm when performing self-supervised learning with ResNet-18. \textit{I-VNE$^+$} will be explained further in Section~\ref{sec:ssl}.  
}
\vspace{-0.5cm}
\end{figure}

\vspace{-0.2cm}

\section{General Applicability of VNE: Experiments}
\label{sec:experiments}
\vspace{-0.1cm}

In this section, we demonstrate the general applicability of VNE by investigating some of the existing representation learning tasks.
Although the results for meta-learning, self-supervised learning (SSL), and GAN can be supported by the theoretical connections between VNE and the popular representation properties presented in Section~\ref{sec:propertiesofvne},
result for domain generalization (DG) is quite surprising.
We will discuss the fundamental difference of DG in Section~\ref{sec:discussion_dg}.

\vspace{-0.1cm}
\subsection{Domain Generalization: Enhancing Generalization}
\label{sec:domain_generalization}
\vspace{-0.1cm}

Given multi-domain datasets, domain generalization attempts to train models that predict well on unseen data distributions~\cite{arjovsky2019invariant}.
In this section, we demonstrate the effectiveness of VNE on ERM~\cite{gulrajani2020search}, one of the most competitive algorithms in DomainBed~\cite{gulrajani2020search}, and on SWAD~\cite{cha2021swad}, which is the state-of-the-art algorithm.
To reproduce the algorithms, we train ERM and SWAD based on an open source in \cite{gulrajani2020search, cha2021swad}. VNE is calculated for the penultimate representation of ResNet-50 models.
Our experiments are performed in leave-one-domain-out setting~\cite{gulrajani2020search} with the most popular datasets (PACS\cite{li2017deeper}, VLCS\cite{fang2013unbiased}, OfficeHome\cite{venkateswara2017deep}, and TerraIncognita\cite{beery2018recognition}).

We have analyzed the eigenvalue distribution of $\mathcal{C}_{\text{auto}}$, and the results are presented in Figure~\ref{figure:analysis_singular_dg}.
At first glance, VNE$^+$ and VNE$^-$ successfully make the eigenvalue distribution more uniform and less uniform, respectively in Figure~\ref{figure:analysis_singular_dg}(a). The corresponding von Neumann entropies are certainly increased by VNE$^+$ and decreased by VNE$^-$ in Figure~\ref{figure:analysis_singular_dg}(b).
When we take a deeper look at the eigenvalues of Vanilla (we count the number of eigenvalues larger than 1e-4), we observe that DG naturally utilizes a small number of eigenvalues (3\% of total).
In this context, we can hypothesize that DG prefers utilizing a relatively small number of dimensions. In addition, the empirical results support the hypothesis.
In Table~\ref{tab:dg_erm}, VNE$^-$ improves all the benchmarks trained with ERM algorithm in four popular datasets.
In Table~\ref{tab:dg_swad}, VNE$^-$ also improves all the benchmarks trained with SWAD algorithm. Furthermore, the resulting performance is the state-of-the-art because SWAD is the current state-of-the-art algorithm.

\begin{figure}[t!]
  \centering
  \subfloat[Eigenvalue distribution (ordered)]{
    \includegraphics[width=0.47\columnwidth]{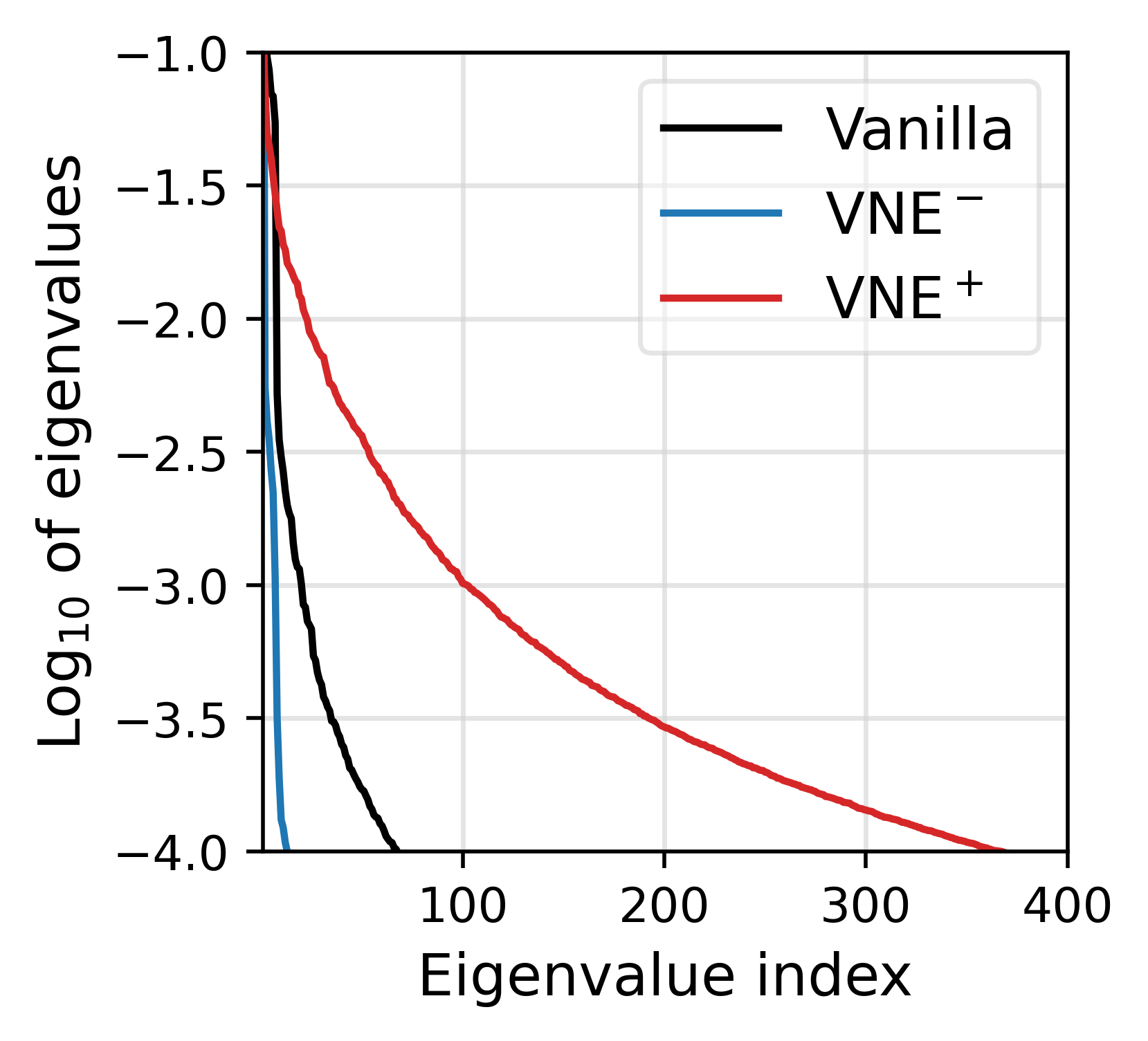}
  }
  \subfloat[von Neumann entropy]{
    \includegraphics[width=0.43\columnwidth]{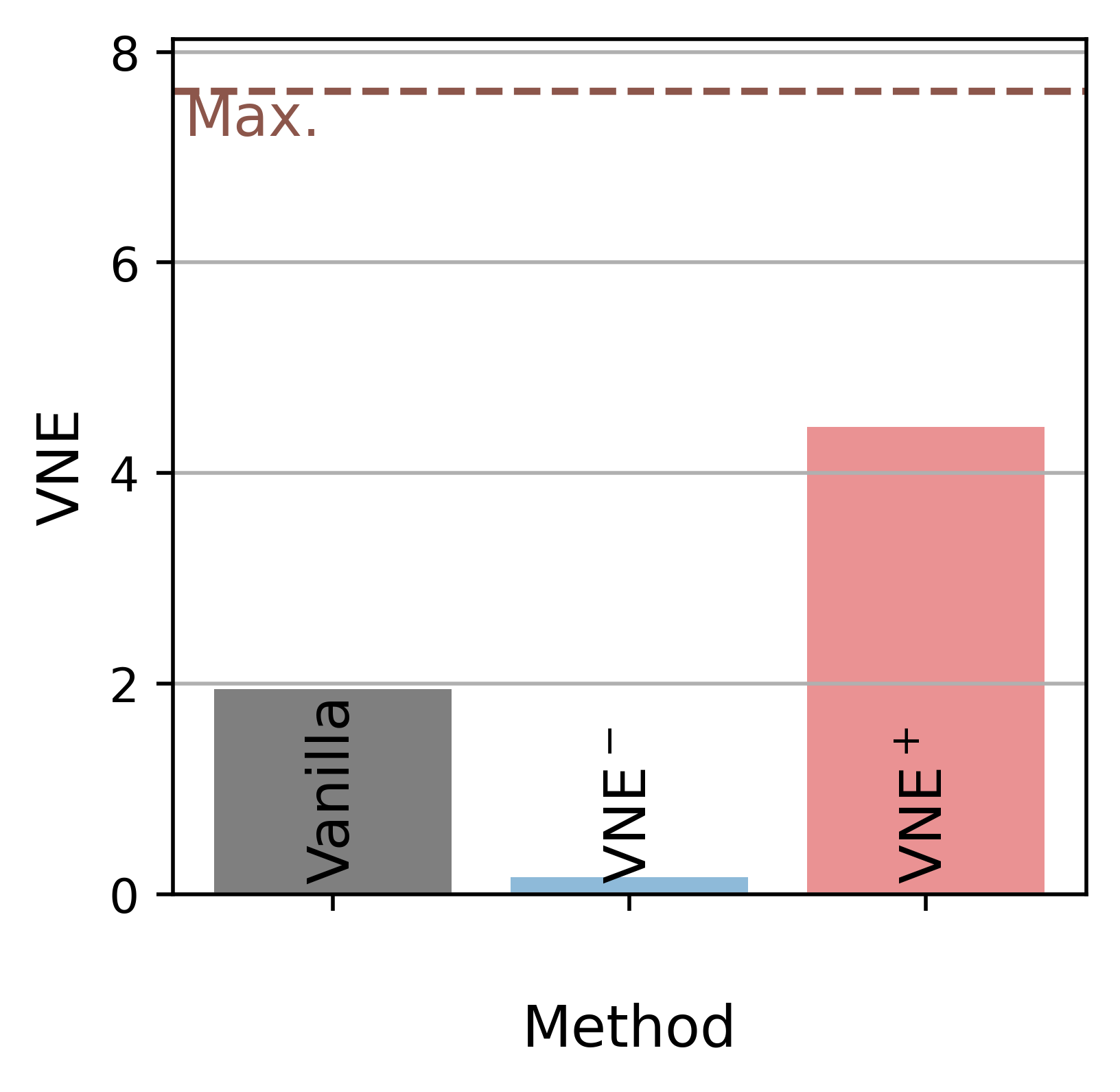}
  }
\vspace{-0.2cm}
\caption{\label{figure:analysis_singular_dg}
Domain Generalization:
In (a), Eigenvalues of $\mathcal{C}_{\text{auto}}$ are presented in log$_{10}$ scale and descending order.
In (b), von Neumann entropies are presented.
ResNet-50 encoders are trained by the ERM~\cite{gulrajani2020search} algorithm with the PACS dataset for 5000 steps. Each representation has a dimension of 2048.
}
\vspace{-0.4cm}
\end{figure}

\begin{table}[ht!]
\centering
\resizebox{0.85\columnwidth}{!}{
\begin{tabular}{@{}llcccccc@{}}
\toprule
Dataset    & \multicolumn{1}{l}{Method} & \multicolumn{4}{c}{Accuracy per test domain}                                                           & \multicolumn{1}{l}{Avg.} & \multicolumn{1}{r}{Diff.}        \\
\midrule
PACS       & \multicolumn{1}{l}{}       & \multicolumn{1}{l}{A}    & \multicolumn{1}{l}{C}   & \multicolumn{1}{l}{P}   & \multicolumn{1}{l}{S}   & \multicolumn{1}{l}{}    &                                  \\
\cmidrule(l){2-2} \cmidrule(l){3-6} \cmidrule(l){7-8}
           & Vanilla                        & 87.6                     & 79.7                    & 95.9                    & 77.6                    & 85.2                    &                                  \\
           & VNE$^+$                   & 82.4                     & 79.2                    & 96.6                    & 70.9                    & 82.3                    & -2.9                             \\
           & VNE$^-$                   & 88.6                     & 79.9                    & 96.7                    & 82.3                    & \textbf{86.9}           & \multicolumn{1}{r}{\textbf{1.7}} \\
\midrule
VLSC       & \multicolumn{1}{l}{}       & \multicolumn{1}{l}{C}    & \multicolumn{1}{l}{L}   & \multicolumn{1}{l}{S}   & \multicolumn{1}{l}{V}   & \multicolumn{1}{l}{}    &                                  \\
\cmidrule(l){2-2} \cmidrule(l){3-6} \cmidrule(l){7-8}
           & Vanilla                        & 98.9                     & 61.5                    & 70.3                    & 76.1                    & 76.7                    &                                  \\
           & VNE$^+$                   & 96.6                     & 65.5                    & 70.1                    & 75.2                    & 76.8                    & 0.1                              \\
           & VNE$^-$                   & 97.5                     & 65.9                    & 70.4                    & 78.4                    & \textbf{78.1}           & \multicolumn{1}{r}{\textbf{1.4}} \\
\midrule
OfficeHome & \multicolumn{1}{l}{}       & \multicolumn{1}{l}{A}    & \multicolumn{1}{l}{C}   & \multicolumn{1}{l}{P}   & \multicolumn{1}{l}{R}   & \multicolumn{1}{l}{}    &                                  \\
\cmidrule(l){2-2} \cmidrule(l){3-6} \cmidrule(l){7-8}
           & Vanilla                        & 57.9                     & 52.5                    & 75.5                    & 73.5                    & 64.9                    &                                  \\
           & VNE$^+$                   & 59.6                     & 50.7                    & 73.1                    & 74.4                    & 64.4                    & -0.5                             \\
           & VNE$^-$                   & 60.4                     & 54.7                    & 73.7                    & 74.7                    & \textbf{65.9}           & \multicolumn{1}{r}{\textbf{1.0}} \\
\midrule
TerraIncognita      & \multicolumn{1}{l}{}       & \multicolumn{1}{l}{L100} & \multicolumn{1}{l}{L38} & \multicolumn{1}{l}{L43} & \multicolumn{1}{l}{L46} & \multicolumn{1}{l}{}    &                                  \\
\cmidrule(l){2-2} \cmidrule(l){3-6} \cmidrule(l){7-8}
           & Vanilla                        & 50.4                     & 42.0                    & 56.8                    & 32.3                    & 45.4                    &                                  \\
           & VNE$^+$                   & 50.3                     & 38.1                    & 55.4                    & 33.6                    & 44.3                    & -1.1                             \\
           & VNE$^-$                   & 58.1                     & 42.9                    & 58.1                    & 43.5                    & \textbf{50.6}           & \multicolumn{1}{r}{\textbf{5.2}} \\ \bottomrule
\end{tabular}
}
\vspace{-0.21cm}
\caption{\label{tab:dg_erm}Domain Generalization: Performance evaluation of models trained with ERM algorithm and various datasets.}
\vspace{-0.17cm}
\end{table}

\begin{table}[ht!]
\centering
\resizebox{0.85\columnwidth}{!}{
\begin{tabular}{@{}llcccccc@{}}
\toprule
Dataset    & \multicolumn{1}{l}{Method} & \multicolumn{4}{c}{Accuracy per test domain}                                                           & \multicolumn{1}{l}{Avg.} & \multicolumn{1}{r}{Diff.}        \\
\midrule
PACS       & \multicolumn{1}{l}{}       & \multicolumn{1}{l}{A}    & \multicolumn{1}{l}{C}   & \multicolumn{1}{l}{P}   & \multicolumn{1}{l}{S}   & \multicolumn{1}{l}{}    &                                  \\
\cmidrule(l){2-2} \cmidrule(l){3-6} \cmidrule(l){7-8}
           & Vanilla                       & 89.2                     & 83.3                    & 97.9                    & 82.5                    & 88.2                    &                                  \\
           & VNE$^+$                   & 87.9                     & 80.6                    & 97.3                    & 78.8                    & 86.2                    & -2.1                             \\
           & VNE$^-$                   & 90.1                     & 83.8                    & 97.5                    & 81.8                    & \textbf{88.3}           & \multicolumn{1}{r}{\textbf{0.1}} \\
\midrule
VLCS       & \multicolumn{1}{l}{}       & \multicolumn{1}{l}{C}    & \multicolumn{1}{l}{L}   & \multicolumn{1}{l}{S}   & \multicolumn{1}{l}{V}   & \multicolumn{1}{l}{}    &                                  \\
\cmidrule(l){2-2} \cmidrule(l){3-6} \cmidrule(l){7-8}
           & Vanilla                       & 98.9                     & 64.5                    & 74.6                    & 79.7                    & 79.4                    &                                  \\
           & VNE$^+$                   & 98.7                     & 62.9                    & 74.9                    & 80.5                    & 79.2                    & -0.2                             \\
           & VNE$^-$                   & 99.2                     & 63.7                    & 74.4                    & 81.6                    & \textbf{79.7}           & \multicolumn{1}{r}{\textbf{0.3}} \\
\midrule
OfficeHome & \multicolumn{1}{l}{}       & \multicolumn{1}{l}{A}    & \multicolumn{1}{l}{C}   & \multicolumn{1}{l}{P}   & \multicolumn{1}{l}{R}   & \multicolumn{1}{l}{}    &                                  \\
\cmidrule(l){2-2} \cmidrule(l){3-6} \cmidrule(l){7-8}
           & Vanilla                       & 64.6                     & 57.7                    & 78.4                    & 80.1                    & 70.2                    &                                  \\
           & VNE$^+$                   & 65.3                     & 57.6                    & 78.6                    & 80.5                    & 70.5                    & 0.3                              \\
           & VNE$^-$                   & 66.6                     & 58.6                    & 78.9                    & 80.5                    & \textbf{71.1}           & \multicolumn{1}{r}{\textbf{0.9}} \\
\midrule
TerraIncognita      & \multicolumn{1}{l}{}       & \multicolumn{1}{l}{L100} & \multicolumn{1}{l}{L38} & \multicolumn{1}{l}{L43} & \multicolumn{1}{l}{L46} & \multicolumn{1}{l}{}    &                                  \\
\cmidrule(l){2-2} \cmidrule(l){3-6} \cmidrule(l){7-8}
           & Vanilla                       & 58.2                     & 45.1                    & 60.9                    & 39.4                    & 50.9                    &                                  \\
           & VNE$^+$                   & 45.3                     & 37.7                    & 60.7                    & 40.5                    & 46.1                    & -4.8                             \\
           & VNE$^-$                   & 59.9                     & 45.5                    & 59.6                    & 41.9                    & \textbf{51.7}           & \multicolumn{1}{r}{\textbf{0.8}} \\ \bottomrule
\end{tabular}
}
\vspace{-0.21cm}
\caption{\label{tab:dg_swad}Domain Generalization: Performance evaluation of models trained with SWAD algorithm and with various datasets. State-of-the-art performances are indicated in bold.}
\vspace{-0.5cm}
\end{table}

\subsection{Meta-Learning: Enhancing Generalization}
\label{sec:meta_learning}

Given meta tasks during the meta-training phase, meta-learning attempts to train meta learners that can generalize well to unseen tasks with just few examples during the meta-testing phase.
In this section, we present the effectiveness of VNE on the most prevalent meta-learning algorithms - MAML~\cite{finn2017model}, ANIL~\cite{raghu2019rapid}, BOIL~\cite{oh2020boil}, ProtoNet~\cite{snell2017prototypical}, Baseline~\cite{chen2019closer}, and Baseline++~\cite{chen2019closer}.
To reproduce the algorithms, we train Baseline, Baseline++, and ProtoNet based on an open source code base in~\cite{chen2019closer} and train MAML, ANIL, and BOIL using torchmeta~\cite{deleu2019torchmeta}. VNE is calculated for the penultimate representation of the standard 4-ConvNet models.
Our experiments are performed in 5-way 1-shot and in 5-way 5-shot with the mini-ImageNet~\cite{vinyals2016matching}, a standard benchmark dataset in few-shot learning.

Similar to domain generalization, we have analyzed the eigenvalue distribution of $\mathcal{C}_{\text{auto}}$, and the results are presented in Figure~\ref{figure:analysis_singular_fsl}.
At first glance, VNE$^+$ and VNE$^-$ successfully make the eigenvalue distribution more uniform and less uniform, respectively in Figure~\ref{figure:analysis_singular_fsl}(a). The corresponding von Neumann entropies are certainly increased by VNE$^+$ and decreased by VNE$^-$ in Figure~\ref{figure:analysis_singular_fsl}(b).
When we take a deeper look at the eigenvalues of Vanilla, we observe that meta-learning naturally utilizes a large number of eigenvalues (94\% of total).
In this context, we can hypothesize that meta-learning prefers utilizing a relatively large number of dimensions. In addition, the hypothesis is supported by the empirical results
where all of six popular benchmark algorithms in both 5-way 1-shot and 5-way 5-shot settings are improved by VNE$^+$ in Table~\ref{tab:meta_learning}. Note that VNE$^+$ consistently provides a gain for all the meta-learning benchmarks that we have investigated.

\begin{figure}[t!]
  \centering
  \subfloat[Eigenvalue distribution (ordered)]{
    \includegraphics[width=0.47\columnwidth]{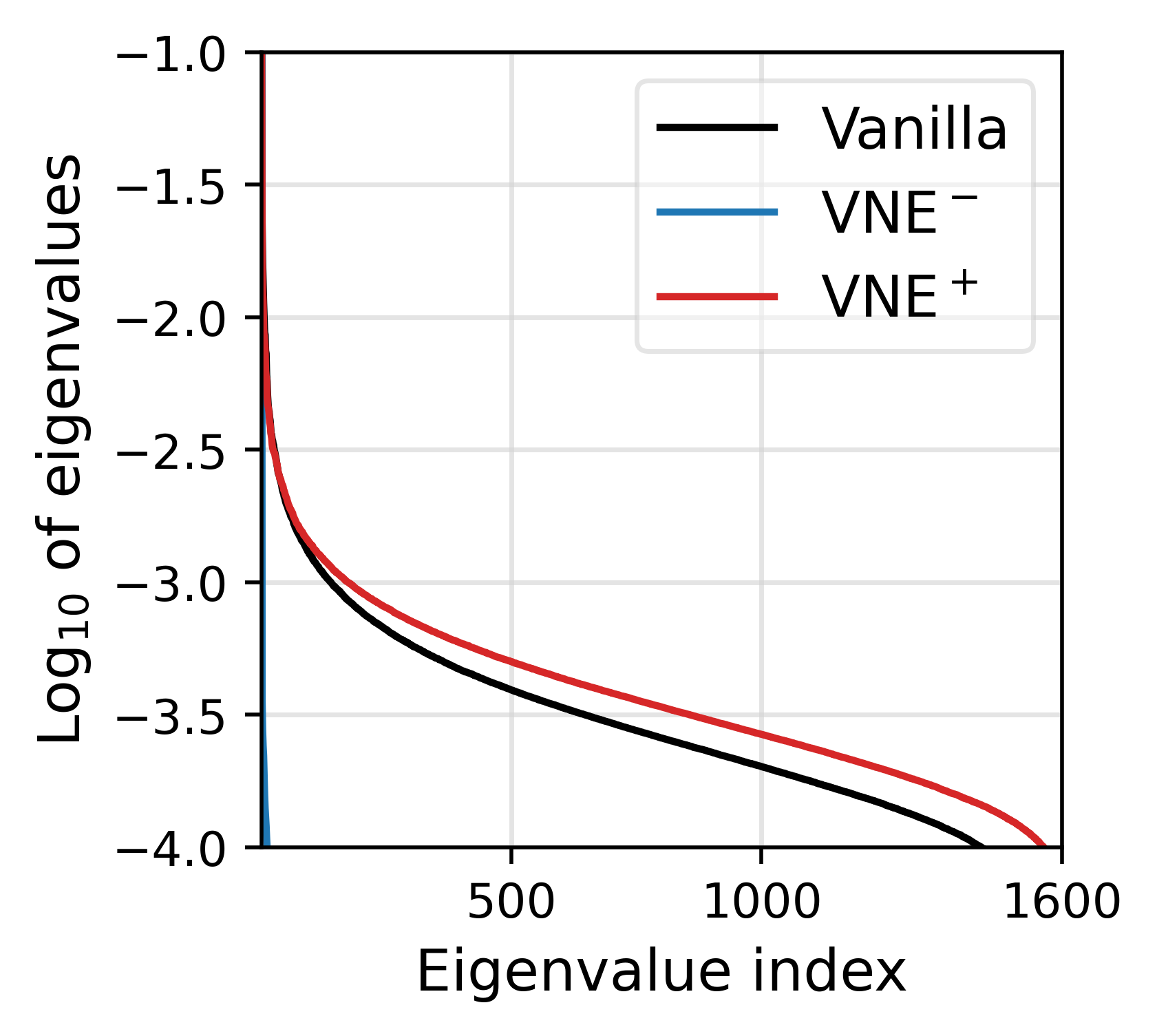}
  }
  \subfloat[von Neumann entropy]{
    \includegraphics[width=0.43\columnwidth]{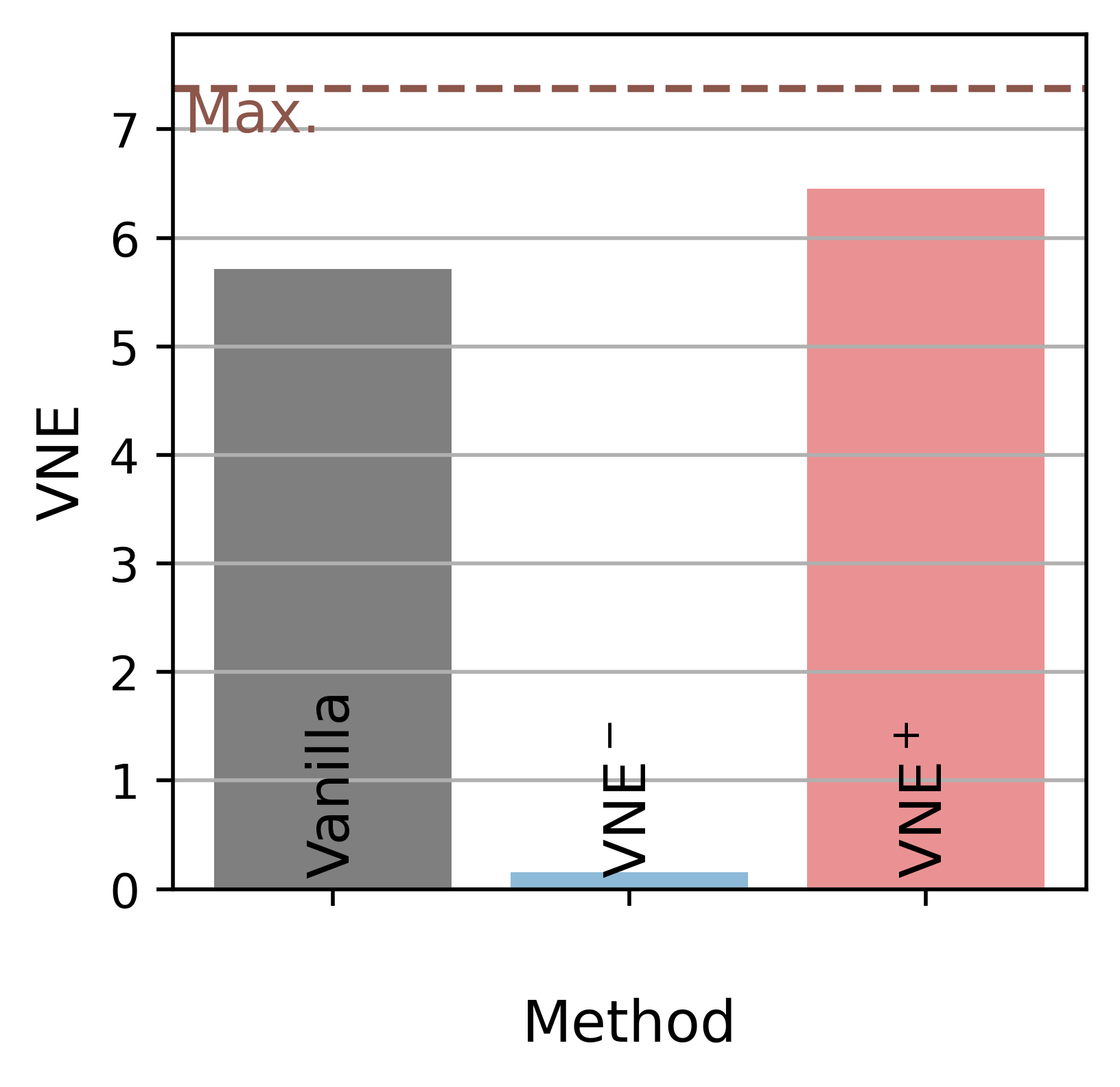}
  }
\vspace{-0.2cm}
\caption{\label{figure:analysis_singular_fsl}
Meta-learning: 
In (a), Eigenvalues of $\mathcal{C}_{\text{auto}}$ are presented in log$_{10}$ scale and descending order.
In (b), von Neumann entropies are presented.
4-ConvNet models are trained by the Baseline~\cite{chen2019closer} algorithm with mini-ImageNet for 100 epochs. Each representation has a dimension of 1600.
}
\vspace{-0.4cm}
\end{figure}

\begin{table}[h!]
\centering
\resizebox{0.85\columnwidth}{!}{
\begin{tabular}{@{}llcccc@{}}
\toprule
Method                                                      &              & \multicolumn{2}{c}{1-shot} & \multicolumn{2}{c}{5-shot} \\ \cmidrule(l){3-4} \cmidrule(l){5-6}
                                                            &              & Avg. Acc. (\%)         & Diff.  & Avg. Acc. (\%)      & Diff.     \\ \midrule
MAML~\cite{finn2017model}             & Vanilla & 48.86 $\pm$ 0.82   &                & 64.59 $\pm$ 0.88 &               \\
                                      & VNE$^-$ & 46.84 $\pm$ 0.76   & -2.02          & 62.57 $\pm$ 0.76 & -2.02         \\
                                      & VNE$^+$ & 50.14 $\pm$ 0.77   & \textbf{ 1.28}  & 66.42 $\pm$ 0.57 & \textbf{ 1.83} \\ \midrule
ANIL~\cite{raghu2019rapid}            & Vanilla & 46.70 $\pm$ 0.40   &                & 61.50 $\pm$ 0.50 &               \\
                                      & VNE$^-$ & 45.40 $\pm$ 0.52   & -1.30          & 60.14 $\pm$ 0.56 & -1.36         \\
                                      & VNE$^+$ & 48.20 $\pm$ 0.45   & \textbf{ 1.50}  & 63.42 $\pm$ 0.45 & \textbf{ 1.92} \\ \midrule
BOIL~\cite{oh2020boil}                & Vanilla & 49.61 $\pm$ 0.16   &                & 66.46 $\pm$ 0.37 &               \\
                                      & VNE$^-$ & 48.42 $\pm$ 0.34   & -1.19          & 65.34 $\pm$ 0.45 & -1.12         \\
                                      & VNE$^+$ & 50.95 $\pm$ 0.42   & \textbf{ 1.34}  & 67.52 $\pm$ 0.46 & \textbf{ 1.06} \\ \midrule
Baseline~\cite{chen2019closer}        & Vanilla & 45.41 $\pm$ 0.72   &                & 62.53 $\pm$ 0.69 &               \\
                                      & VNE$^-$ & 30.43 $\pm$ 0.72   & -14.98         & 48.03 $\pm$ 0.90 & -14.50        \\
                                      & VNE$^+$ & 47.03 $\pm$ 0.73   & \phantom{-0}\textbf{1.62}  & 65.85 $\pm$ 0.67 & \phantom{-0}\textbf{3.32} \\ \midrule
Baseline++~\cite{chen2019closer}      & Vanilla & 47.95 $\pm$ 0.74   &                & 66.43 $\pm$ 0.63 &               \\
                                      & VNE$^-$ & 29.52 $\pm$ 0.76   & -18.43         & 60.98 $\pm$ 0.78 & -5.45         \\
                                      & VNE$^+$ & 50.17 $\pm$ 0.77   & \phantom{-0}\textbf{2.22}  & 67.25 $\pm$ 0.67 & \textbf{ 0.82} \\ \midrule
ProtoNet~\cite{snell2017prototypical} & Vanilla & 43.16 $\pm$ 0.55   &                & 64.24 $\pm$ 0.72 &           \\
                                      & VNE$^-$ & -                  &                & 62.14 $\pm$ 0.69 & -2.10          \\
                                      & VNE$^+$ & 46.81 $\pm$ 0.35   & \textbf{ 3.65}  & 66.72 $\pm$ 0.71 & \textbf{ 2.48}          \\ \bottomrule
\end{tabular}
}
\vspace{-0.2cm}
\caption{\label{tab:meta_learning}Meta-learning: Performance evaluation of 5-way 1-shot and 5-way 5-shot with mini-ImageNet.}
\vspace{-0.5cm}
\end{table}

\subsection{SSL: Preventing Representation Collapse}
\label{sec:ssl}

Given an unlabelled dataset, self-supervised learning attempts to learn representation that makes various downstream tasks easier.
In this section, we demonstrate the effectiveness of VNE on self-supervised learning by proposing a novel method called \textit{I-VNE$^+$} where \textbf{I}nvariant loss is simply implemented by maximizing cosine similarity between positive pairs while consequent representation collapse is prevented by \textbf{VNE$^+$}. The loss is expressed as:
\vspace{-0.1cm}
\begin{equation}
\label{eq:ivne}
\mathcal{L}_{\text{I-VNE}^+}=-\alpha_1 \cdot \mathbb{E}_i[\text{sim}(\bm{h}^1_i,\bm{h}^2_i)]-\alpha_2 \cdot S(\mathcal{C}_{\text{auto}}),
\vspace{-0.1cm}
\end{equation}
where sim$(\bm{h}^1_i,\bm{h}^2_i)$ indicates the cosine similarity between two $i$th row vectors, $\bm{h}^1_i$ and $\bm{h}^2_i$, of representation matrices, $\bm{H}_1$ and $\bm{H}_2$, from two views, and $\mathcal{C}_{\text{auto}}$ is calculated for $\bm{H}_1$.
For experiments, we follow the standard training protocols from~\cite{grill2020bootstrap,zbontar2021barlow} (Refer to Supplementary~\ref{sec:implementation_delail} for more details) and the standard evaluation protocols from~\cite{misra2020self,goyal2019scaling,grill2020bootstrap,zbontar2021barlow}.

\begin{figure}[t!]
  \centering
  \subfloat[Eigenvalue distribution (ordered)]{
    \includegraphics[width=0.47\columnwidth]{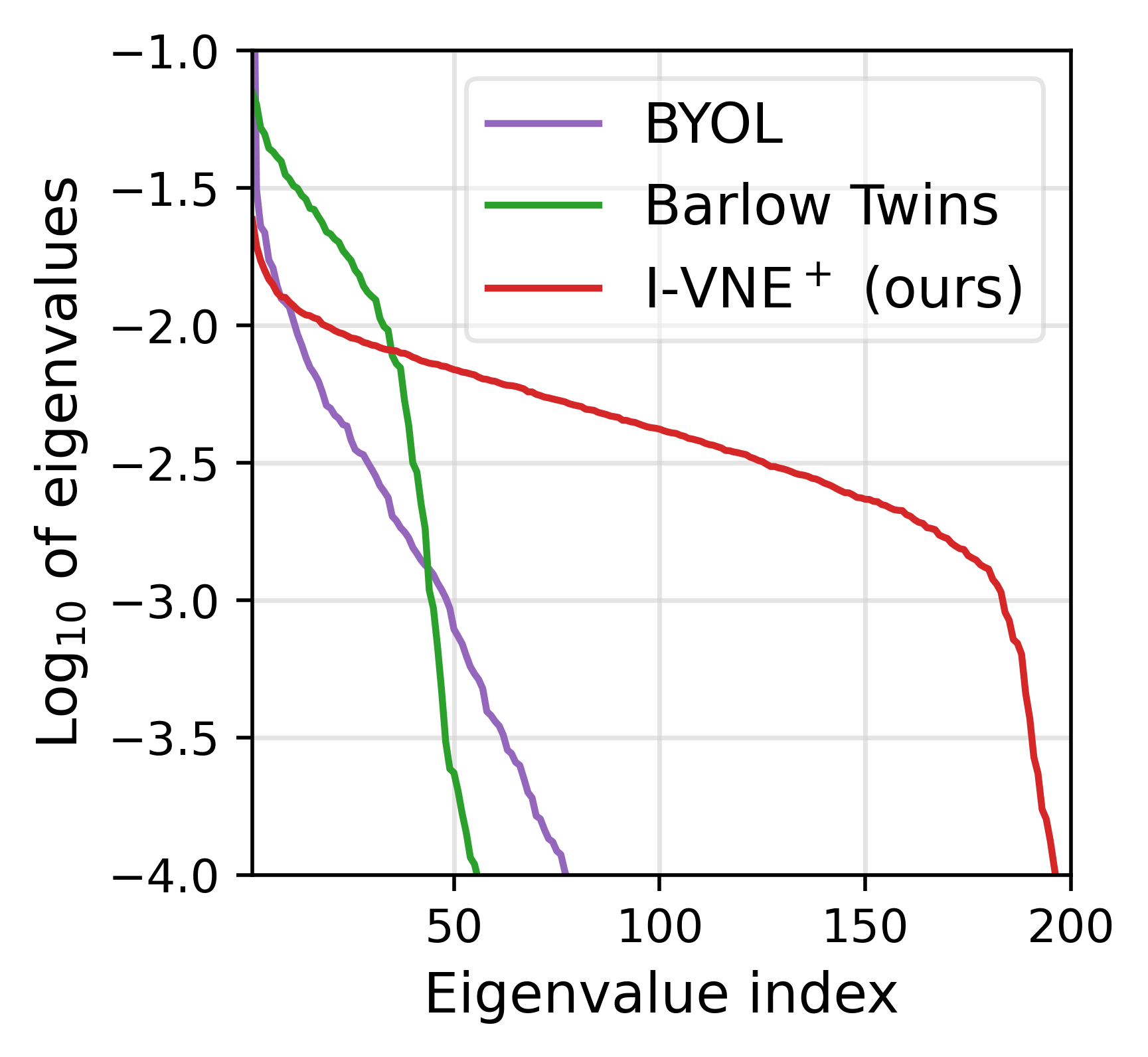}
  }
  \subfloat[von Neumann entropy]{
    \includegraphics[width=0.44\columnwidth]{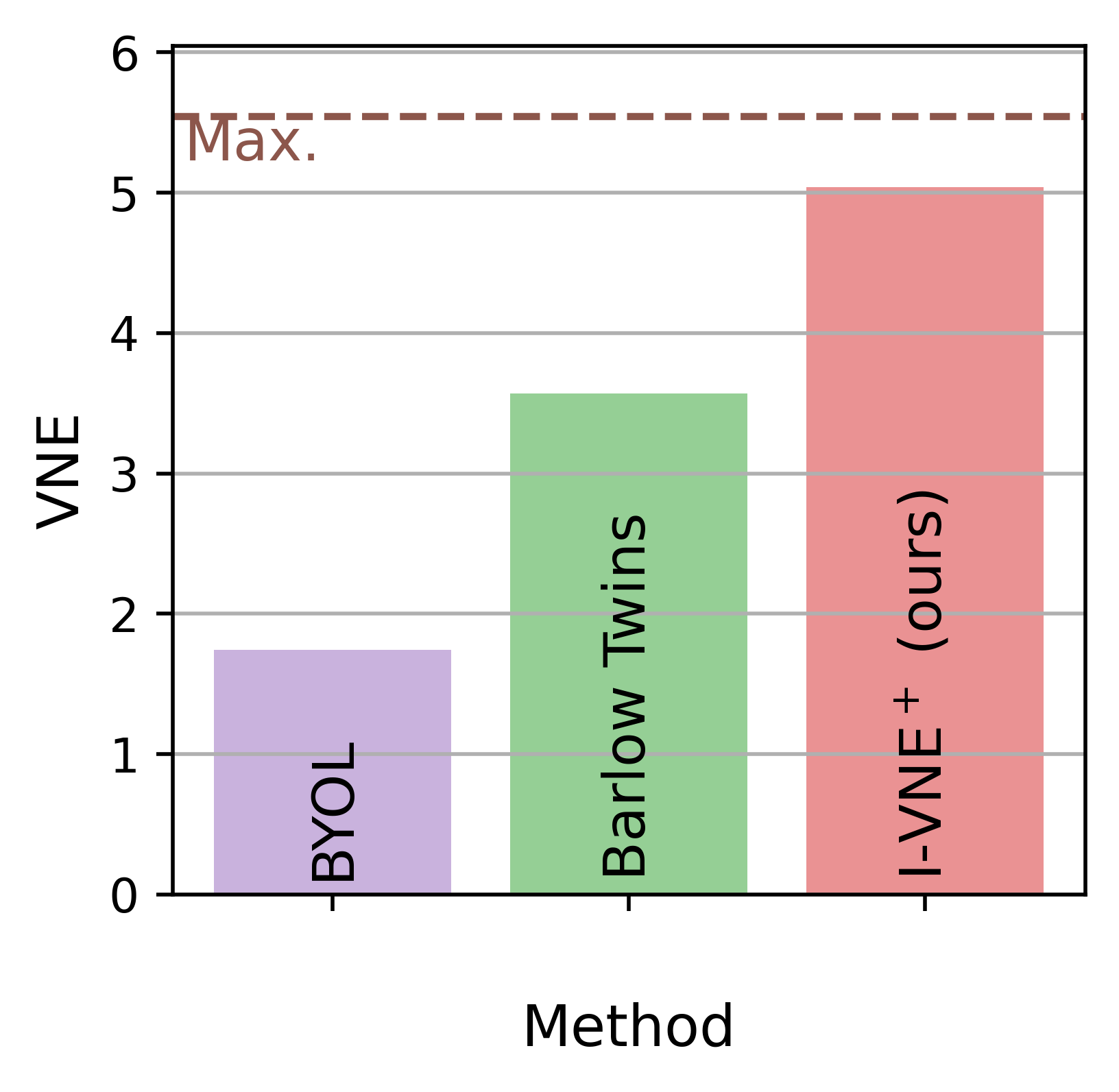}
  }
\vspace{-0.2cm}
\caption{\label{figure:analysis_singular_ssl}
SSL:
In (a), Eigenvalues of $\mathcal{C}_{\text{auto}}$ are presented in log$_{10}$ scale and descending order.
In (b), von Neumann entropies are presented.
ResNet-50 encoders and mlp projectors are trained by BYOL, Barlow Twins, and I-VNE$^+$ methods with ImageNet-100 for 100 epochs. Projectors for all methods have the same output dimension of 256.
}
\vspace{-0.2cm}
\end{figure}

In fact, the loss formulation in Eq.~\eqref{eq:ivne} is equivalent to a simple combination of the loss term in BYOL~\cite{grill2020bootstrap} and VNE$^+$ without predictor and stop gradient. $S(\mathcal{C}_{\text{auto}})$ term in I-VNE$^+$ can also replace the redundancy reduction term in Barlow Twins~\cite{zbontar2021barlow}.
Therefore, we have analyzed the eigenvalue distribution of $\mathcal{C}_{\text{auto}}$ and $S(\mathcal{C}_{\text{auto}})$ by comparing with BYOL and Barlow Twins, and the results are presented in Figure~\ref{figure:analysis_singular_ssl}. Simply put, I-VNE$^+$ utilizes more eigenvalues of $\mathcal{C}_{\text{auto}}$ and has a larger value of $S(\mathcal{C}_{\text{auto}})$ than the others.
Because I-VNE$^+$ utilizes more eigen-dimensions of $\mathcal{C}_{\text{auto}}$ than the others,
the dimensional collapse problem prevailing in SSL~\cite{hua2021feature,jing2021understanding} can be mitigated by I-VNE$^+$; hence better performance with I-VNE$^+$ can be expected.

\begin{table}[t!]
\centering
\subfloat[CIFAR-10]{
    \resizebox{.493\columnwidth}{!}{
\begin{tabular}{@{}lcc@{}}
\toprule
Method                                            & Epoch & Top-1 \\ \midrule
Supervised \cite{chen2020simple} &       & 95.1  \\ \midrule
NT-Xent \cite{chen2021intriguing}               & 200   & 91.3  \\
Decoupled NT-Xent \cite{chen2021intriguing}     & 200   & 91.3  \\
SWD \cite{chen2021intriguing}                   & 200   & 90.8  \\
NT-Xent \cite{chen2021intriguing}               & 800   & 93.9  \\
Decoupled NT-Xent \cite{chen2021intriguing}     & 800   & 94.0  \\
SWD \cite{chen2021intriguing}                   & 800   & 94.1  \\ 
Shuffled-DBN \cite{hua2021feature}              & 200   & 89.5  \\ \midrule
I-VNE$^+$ (ours)                                        & 200   & \textbf{94.3}  \\
I-VNE$^+$ (ours)                                        & 400   & \textbf{94.3}  \\ \bottomrule
\end{tabular}
}
}
\hspace{-0.2cm}
\subfloat[ImageNet-100]{
    \resizebox{.477\columnwidth}{!}{
\begin{tabular}{@{}lcc@{}}
\toprule
Method                                                      & Epoch & Top-1 \\ \midrule
Supervised \cite{kalantidis2020hard}       &       & 86.2        \\ \midrule
Align.+Uniform. \cite{wang2020understanding} & 240   & 74.6        \\
CMC (K=1) \cite{zheng2021contrastive}      & 200   & 75.8        \\
CMC (K=4) \cite{zheng2021contrastive}      & 200   & 78.8        \\
CACR(K=1) \cite{zheng2021contrastive}      & 200   & 79.4        \\
CACR(K=4) \cite{zheng2021contrastive}      & 200   & 80.5        \\
LooC++ \cite{xiao2020should}               & 500   & 82.2        \\
MoCo-v2+MoCHi \cite{kalantidis2020hard}    & 800   & 84.5        \\ \midrule
I-VNE$^+$ (ours)                           & 200   & 84.7        \\
I-VNE$^+$ (ours)                           & 800   & \textbf{86.3}        \\ \bottomrule
\end{tabular}
}
}
\vspace{-0.2cm}
\caption{\label{tab:ssl_linear}SSL: Linear evaluation performance for various representation learning methods. They are all based on ResNet-50 encoders pre-trained with various datasets. Linear classifier on top of the frozen pre-trained model is trained with labels. State-of-the-art methods are included and the best results are indicated in bold.
}
\vspace{-0.3cm}
\end{table}

To evaluate I-VNE$^+$, we compare benchmark performance with prior state-of-the-art methods.
In Table~\ref{tab:ssl_linear}(a) and (b),
I-VNE$^+$ outperforms prior state-of-the-art linear evaluation benchmarks in both CIFAR-10 and ImageNet-100. Moreover, I-VNE$^+$ even surpasses the supervised performance in ImageNet-100.
In ImageNet-1K, I-VNE$^+$ shows competitive linear evaluation performance which is above the average (71.8\%) as demonstrated in Table~\ref{tab:ssl_linear_appendix} of Supplementary~\ref{sec:supplementary_results}.
In addition, we can show that the pre-trained model of ImageNet-1K shows state-of-the-art performance in the following evaluation benchmarks.
In Table~\ref{tab:ssl_semi_supervised}, I-VNE$^+$ outperforms all the semi-supervised learning benchmarks except for Top-1 accuracy with 10\% data regime.
In Table~\ref{tab:ssl_transfer}, I-VNE$^+$ outperforms all the transfer learning benchmarks with COCO.
The results indicate that I-VNE$^+$ is advantageous for more sophisticated tasks such as low-data regime (semi-supervised) and out-of-domain (transfer learning with COCO) tasks.

\begin{table}[t!]
\centering
\resizebox{0.6\columnwidth}{!}{
\begin{tabular}{@{}lcccc@{}}
\toprule
Method                                                 & \multicolumn{2}{c}{Top-1}     & \multicolumn{2}{c}{Top-5}     \\ \cmidrule(lr){2-3}\cmidrule(lr){4-5}
                                                      & 1$\%$         & 10$\%$        & 1$\%$         & 10$\%$        \\ \midrule
Supervised \cite{chen2020simple}      & 25.4          & 56.4          & 48.4          & 80.4          \\ \midrule
SimCLR \cite{chen2020simple}          & 48.3          & 65.6          & 75.5          & 87.8          \\
BYOL \cite{grill2020bootstrap}        & 53.2          & 68.8          & 78.4          & 89.0          \\
SwAV \cite{caron2020unsupervised}     & 53.9          & \textbf{70.2}          & 78.5          & \textbf{89.9}          \\ 
VICReg \cite{bardes2021vicreg}        & 54.8          & 69.5          & 79.4          & 89.5          \\
Barlow Twins \cite{zbontar2021barlow} & 55.0          & 69.7          & 79.2          & 89.3          \\ \midrule
I-VNE$^+$ (ours)                            & \textbf{55.8} & 69.1 & \textbf{81.0} & \textbf{89.9} \\ \bottomrule
\end{tabular}
}
\vspace{-0.2cm}
\caption{\label{tab:ssl_semi_supervised}SSL: Semi-supervised learning evaluation performance for various representation learning methods. They are all based on ResNet-50 encoders pre-trained with ImageNet-1K. Pre-trained models are fine-tuned with 1\% and 10\% labels of ImageNet-1K. State-of-the-art methods are included and the best results are indicated in bold.
}
\vspace{-0.2cm}
\end{table}

\begin{table}[t!]
\centering
\resizebox{0.85\columnwidth}{!}{
\begin{tabular}{@{}lcccccc@{}}
\toprule
Method                                                 & \multicolumn{3}{c}{COCO det.}                 & \multicolumn{3}{c}{COCO instance seg.}                                 \\ \cmidrule(lr){2-4}\cmidrule(lr){5-7}
                                                      & AP            & AP$_{50}$     & AP$_{75}$     & AP$^{\text{mask}}$ & AP$_{50}^{\text{mask}}$ & AP$_{75}^{\text{mask}}$ \\ \midrule
Scratch \cite{chen2021exploring}      & 26.4          & 44.0          & 27.8          & 29.3               & 46.9                    & 30.8                    \\
Supervised \cite{chen2021exploring}   & 38.2          & 58.2          & 41.2          & 33.3               & 54.7                    & 35.2                    \\ \midrule
SimCLR \cite{chen2021exploring}       & 37.9          & 57.7          & 40.9          & 33.3               & 54.6                    & 35.3                    \\
BYOL \cite{chen2021exploring}         & 37.9          & 57.8          & 40.9          & 33.2               & 54.3                    & 35.0                    \\
SwAV \cite{zbontar2021barlow}         & 38.4          & 58.6          & 41.3          & 33.8               & 55.2                    & 35.9                    \\
SimSiam \cite{chen2021exploring}      & 39.2          & 59.3 & 42.1          & 34.4 & 56.0 & 36.7  \\
MoCov2 \cite{zbontar2021barlow}       & 39.3 & 58.9 & 42.5  & 34.4  & 55.8                    & 36.5                    \\ 
Barlow Twins \cite{zbontar2021barlow} & 39.2          & 59.0          & 42.5 & 34.3               & 56.0                    & 36.5                    \\ \midrule
I-VNE$^+$ (ours)                                             & \textbf{41.2} & \textbf{61.3} & \textbf{44.6} & \textbf{35.7}      & \textbf{57.9}           & \textbf{38.0}           \\ \bottomrule
\end{tabular}
}
\vspace{-0.2cm}
\caption{\label{tab:ssl_transfer}SSL: Transfer learning evaluation performance for various representation learning methods. They are all based on ResNet-50 encoders pre-trained in ImageNet-1K. Pre-trained models are fine-tuned with COCO detection and instance segmentation tasks using Mask R-CNN with C4-backbone~\cite{he2017mask, wu2019detectron2} and using 1 $\times$ schedule. State-of-the-art methods are included and the best results are indicated in bold.
}
\vspace{-0.5cm}
\end{table}

\subsection{GAN: Preventing Mode Collapse}
\label{sec:gan}

In Section~\ref{sec:ssl}, VNE$^+$ has successfully prevented representation collapse.
As another example for collapse prevention, we consider the mode collapse in GAN.
The GAN training usually ends up with (partial) mode collapse~\cite{goodfellow2016nips}, where generative models suffer lack of diversity.
To demonstrate that this problem can be solved by VNE$^+$,
we reproduce various GAN methods based on an open source code base, StudioGAN~\cite{kang2022StudioGAN} and train all models with CIFAR-10 for 100 epochs.
To evaluate the models, we report the Inception Score~\cite{salimans2016improved} (IS, higher is better) and the Fréchet Inception Distance~\cite{heusel2017gans} (FID, lower is better). Although both IS and FID are the most popular metrics for evaluating generative models, FID is known to favor more diversified images~\cite{brock2018large}.
Table~\ref{tab:sup_gan} demonstrate that the overall quality of the output, especially diversity, has been improved by VNE$^+$ because FID scores have been improved. IS has also been improved.

\begin{table}[t!]
\centering
\resizebox{0.9\columnwidth}{!}{
\begin{tabular}{@{}lcccccc@{}}
\toprule
 & \multicolumn{3}{c}{Inception Score~$\uparrow$}         & \multicolumn{3}{c}{Fréchet Inception Distance~$\downarrow$} \\ \cmidrule(){1-1} \cmidrule(l){2-4} \cmidrule(l){5-7}
Method            & Vanilla & VNE$^+$ & Diff.                 & Vanilla  & VNE$^+$ & Diff.                   \\ \midrule
DCGAN             & 6.49     & 6.74     & $\uparrow$~0.25        & 42.55     & 35.44    & $\downarrow$~7.11        \\
PD-GAN            & 7.83     & 8.01     & $\uparrow$~0.18        & 28.02     & 23.54    & $\downarrow$~4.48        \\
LOGAN             & 8.02     & 8.15     & $\uparrow$~0.13        & 18.88     & 17.17    & $\downarrow$~1.71        \\
WGAN-GP           & 7.37     & 7.42     & $\uparrow$~0.05        & 24.62     & 23.31    & $\downarrow$~1.31        \\
SAGAN             & 8.86     & 8.90     & $\uparrow$~0.04        & \phantom{0}9.55      & \phantom{0}8.91     & $\downarrow$~0.64        \\
SNGAN             & 8.85     & 8.86     & $\uparrow$~0.01        & \phantom{0}9.97      & \phantom{0}9.41     & $\downarrow$~0.56        \\
BigGAN            & 9.82     & 9.83     & $\uparrow$~0.01        & \phantom{0}5.34      & \phantom{0}5.18     & $\downarrow$~0.16        \\ \bottomrule
\end{tabular}
}
\vspace{-0.2cm}
\caption{\label{tab:sup_gan}GAN: Performance evaluation results.
}
\vspace{-0.4cm}
\end{table}

\section{Theoretical Connections of VNE}
\label{sec:propertiesofvne}
\vspace{-0.1cm}

In Section~\ref{sec:vonneumannentropy}, we have examined the popular regularization objective of $\mathcal{C}_{\text{auto}} \rightarrow c \cdot I_d$ and explained how von Neumann entropy can be a desirable regularization method. In addition, von Neumann entropy can be beneficial in a few different ways because of its conceptual connection with conventional representation properties such as rank, disentanglement, and isotropy. In this section, we establish a theoretical connection with each property and provide a brief discussion.

\subsection{Rank of Representation}

The rank of representation, $\text{rank}(\mathcal{C}_{\text{auto}})$, directly measures the number of dimensions utilized by the representation. Von Neumann Entropy in Eq.~(\ref{eq:vne_formula}) is closely related to the rank, where it is maximized when $\mathcal{C}_{\text{auto}}$ is full rank with uniformly distributed eigenvalues 
and it is minimized when $\mathcal{C}_{\text{auto}}$ is rank one. In fact, a formal bound between rank and VNE can be derived. 
\vspace{-0.1cm}
\begin{theorem}[Rank and VNE]
\label{thm:rank}
For a given representation autocorrelation $\mathcal{C}_{\text{auto}} = \bm{H}^T\bm{H}/N \in \mathbb{R}^{d \times d}$ of rank $k$ $(\le d)$,
\vspace{-0.3cm}
\begin{equation}
\text{log}(\text{rank}(\mathcal{C}_{\text{auto}})) \ge S(\mathcal{C}_{\text{auto}}),
\end{equation}
where equality holds iff the eigenvalues of $\mathcal{C}_{\text{auto}}$ are uniformly distributed with $\forall_{j=1}^k \lambda_{j}=1/k$ and $\forall_{j=k+1}^d \lambda_{j}=0$.
\vspace{-0.5cm}
\end{theorem}
Refer to Supplementary~\ref{sec:prop_proof} for the proof. Theorem~\ref{thm:rank} states that $\text{log}(\text{rank}(\mathcal{C}_{\text{auto}}))$ is lower bounded by $S(\mathcal{C}_{\text{auto}})$ and that the bound is tight when non-zero eigenvalues are uniformly distributed. The close relationship between rank and VNE can also be confirmed empirically. For the VNE plots in Figure~\ref{figure:analysis_singular_dg}(b) and Figure~\ref{figure:analysis_singular_fsl}(b), we have compared their rank values and the results are presented in Figure~\ref{figure:analysis_rank}.

Although the rank is a meaningful and useful measure of $\mathcal{C}_{\text{auto}}$, it cannot be directly used for learning because of its discrete nature. In addition, it can be misleading because even extremely small non-zero eigenvalues contribute toward the rank. VNE can be a useful proxy of the rank because it does not suffer from either of the problems. 

\begin{figure}[t!]
\centering
\subfloat[Domain Generalization]{
\includegraphics[width=0.425\columnwidth]{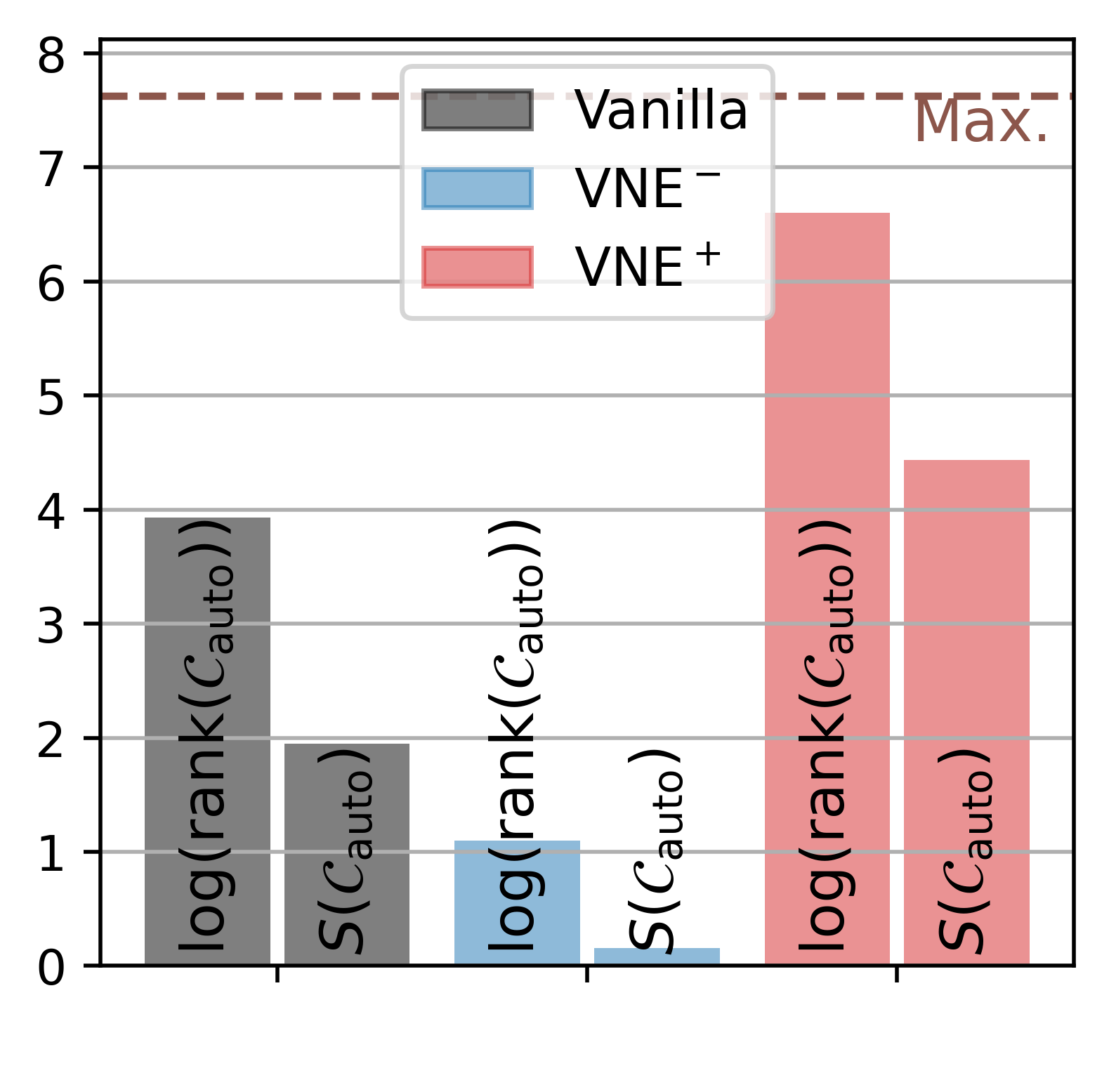}
}
\subfloat[Meta-learning]{
\includegraphics[width=0.425\columnwidth]{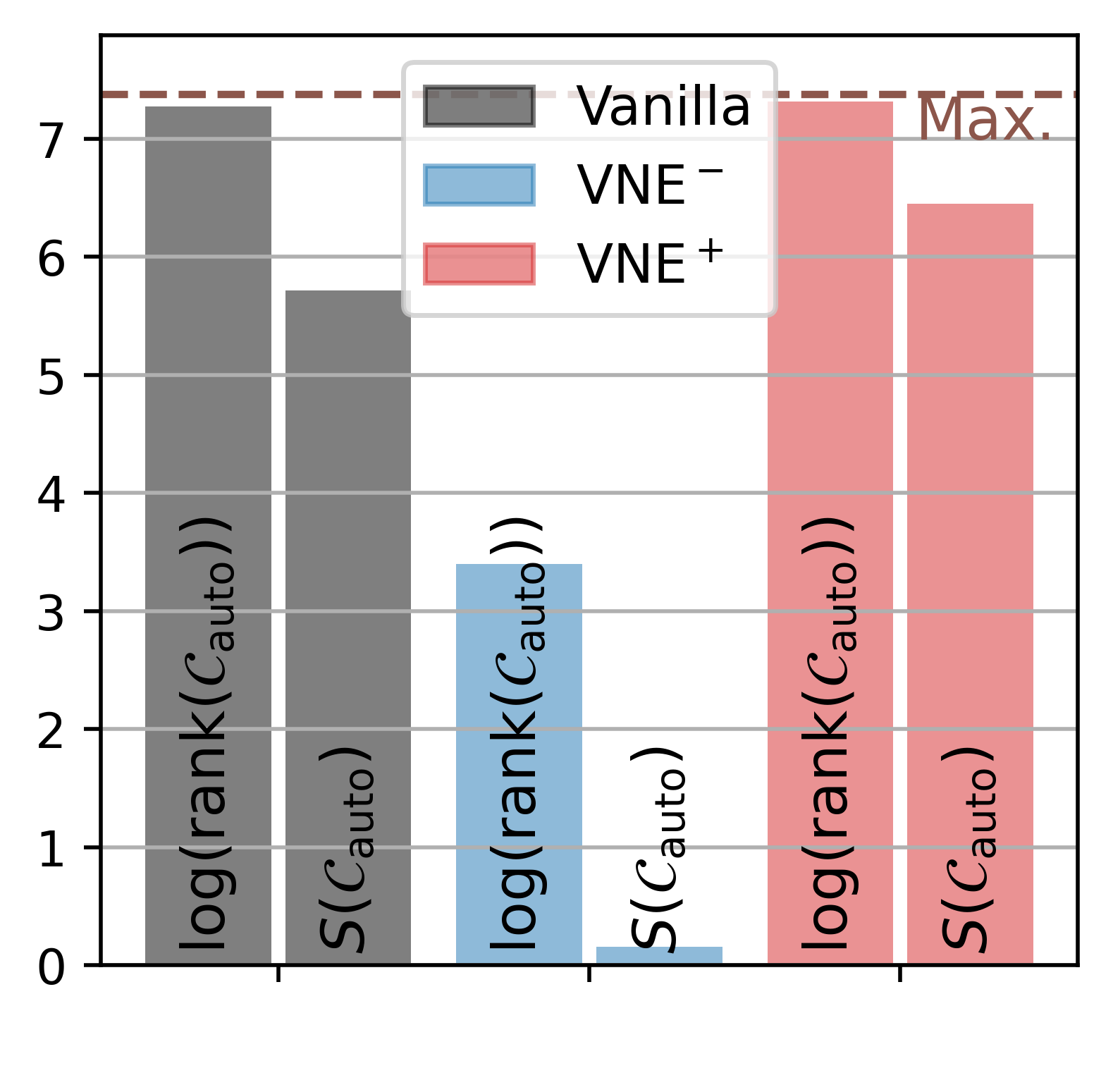}
}
\vspace{-0.2cm}
\caption{\label{figure:analysis_rank} Rank of representation: Comparison between $\text{log}(\text{rank}(\mathcal{C}_{\text{auto}}))$ and $S(\mathcal{C}_{\text{auto}})$.
As a surrogate of the rank, the count of the largest eigenvalues required for the 99\% of total eigenvalue energy is presented.
The possible maximum value, $\log{d}$, is depicted as the brown dotted line.
}
\vspace{-0.45cm}
\end{figure}

\subsection{Disentanglement of Representation}
Although disentanglement has been considered as a desirable property of representation~\cite{bengio2009learning,achille2018emergence}, its formal definition can be dependent on the context of the research. In this study, we adopt the definition in~\cite{achille2018emergence}, where a representation vector $\bm{h}$ is disentangled if its scalar components are independent. To understand the relationship between von Neumann entropy and disentanglement, we derive a theoretical result under a multi-variate Gaussian assumption and provide an empirical analysis. The assumption can be formally described as:
\vspace{-0.1cm}
\begin{assumption}
\label{assumption:zero_mean_gaussian}
We assume that representation $\bm{h}$ follows zero-mean multivariate Gaussian distribution.
In addition, we assume that the components of $\bm{h}$ (denoted as $\bm{h}^{(i)}$) have homogeneous variance of $\frac{1}{d}$, i.e., $\forall_{i=1}^d\bm{h}^{(i)} \sim \mathcal{N} (0, \frac{1}{d})$.
\end{assumption}

The multi-variate Gaussian assumption is not new, and it has been utilized in numerous studies. For instance, \cite{kingma2013auto,lee2017deep,yang2021free} adopted the assumption. In addition, the assumption was proven to be true for infinite width neural networks~\cite{neal1996priors,williams1997computing,neal2012bayesian,lee2017deep}. Numerous studies applied a representation normalization to have a homogeneous variance~(e.g., via batch normalization~\cite{ioffe2015batch}). Under the Assumption~\ref{assumption:zero_mean_gaussian}, our main result can be stated as below.
\begin{theorem}[Disentanglement and VNE]
\label{thm:disentanglement}
Under the Assumption~\ref{assumption:zero_mean_gaussian}, $\bm{h}$ is disentangled if $S(\mathcal{C}_{\text{auto}})$ is maximized.
\end{theorem}
Refer to Supplementary~\ref{sec:prop_proof} for the proof.
Theorem~\ref{thm:disentanglement} states that the Gaussian representation $\bm{h}$ is disentangled if von Neumann entropy $S(\mathcal{C}_{\text{auto}})$ is fully maximized. The theoretical result can also be confirmed with an empirical analysis. For the domain-generalization experiment in Section~\ref{sec:experiments}, we have randomly chosen two components $\bm{h}^{(i)}$ and $\bm{h}^{(j)}$, where $i \ne j$, and compared their cosine similarity for the examples in the mini-batch. The resulting distributions are presented in Figure~\ref{figure:analysis_sim_units}(a). It can be clearly observed that VNE$^+$ makes the linear dependence between 
two components to be significantly weaker (cosine similarity closer to zero) while VNE$^-$ can make it stronger. The same behavior can be observed for a supervised learning example in Figure~\ref{figure:analysis_sim_units}(b). Therefore, the representation components are decorrelated by VNE$^+$ and correlated by VNE$^-$. 
For meta-learning, the trend is the same, but the shift in the distribution turns out to be relatively limited (see Figure~\ref{figure:supple_dis} in Supplementary~\ref{sec:supplementary_results}). 

Similar to the case of rank, von Neumann entropy can be utilized as a proxy for controlling the degree of disentanglement in representation. In the case of supervised learning in Figure~\ref{figure:analysis_sim_units}(b), it can be observed that both highly disentangled and highly entangled representations can be learned by regularizing von Neumann entropy.

\begin{figure}[t!]
\centering
\subfloat[Domain Generalization]{
\includegraphics[width=0.45\columnwidth]{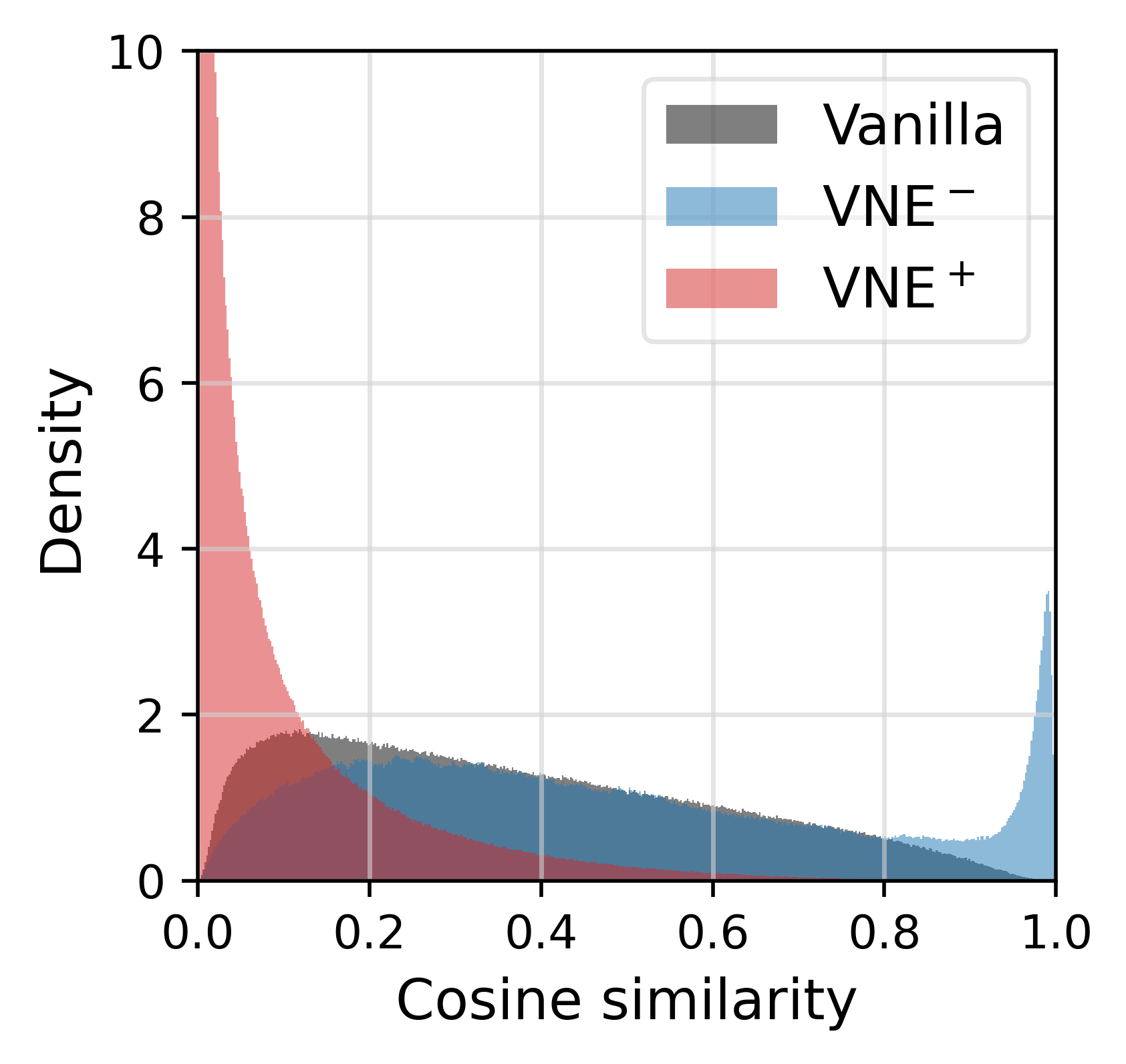}
}
\subfloat[Supervised (ImageNet-100)]{
\includegraphics[width=0.45\columnwidth]{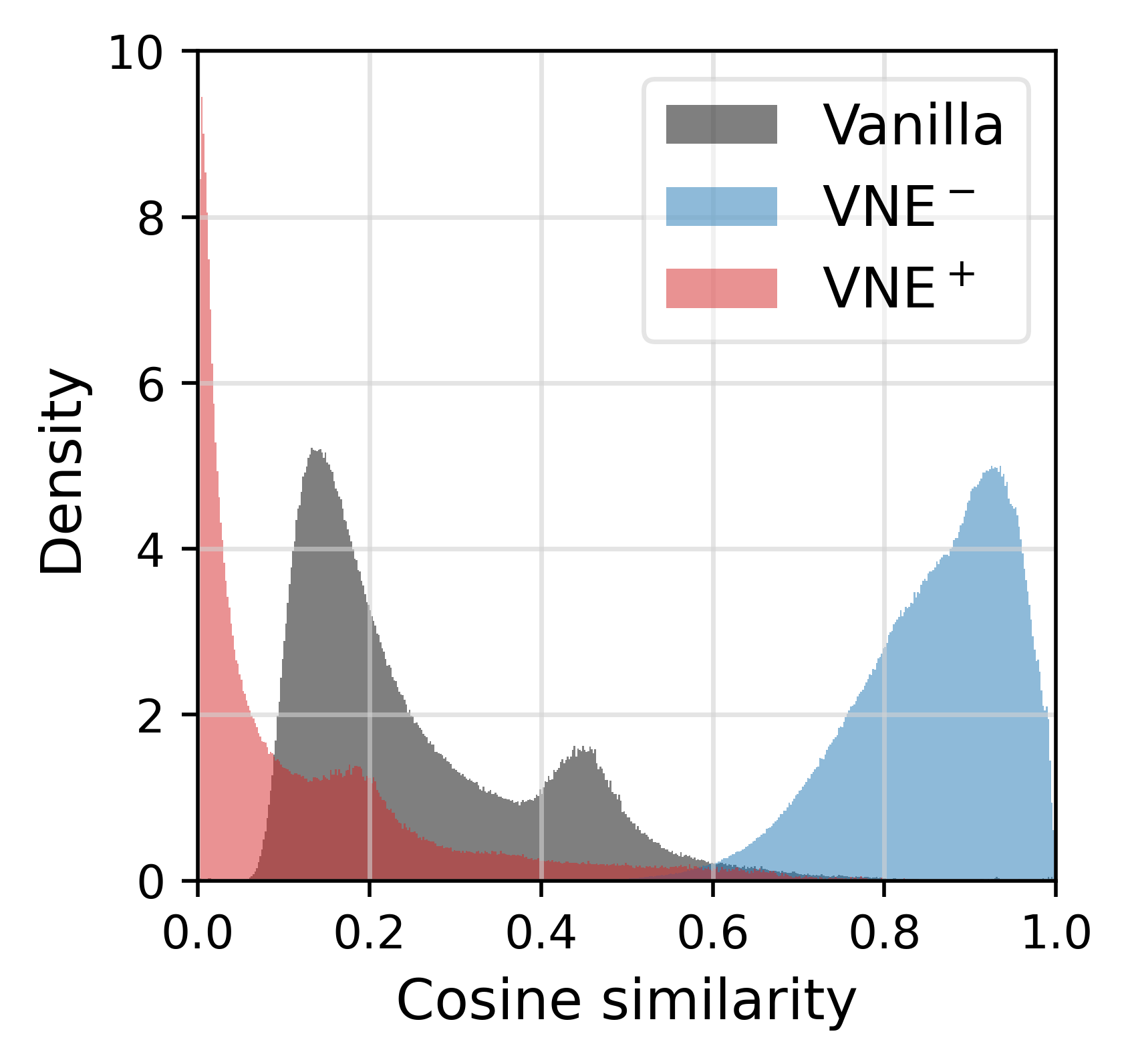}
}
\vspace{-0.2cm}
\caption{\label{figure:analysis_sim_units}Disentanglement of representation: Distribution of cosine similarity between pairwise components in representation. Note that all the values are positive because of the ReLU layer.}
\vspace{-0.4cm}
\end{figure}

\subsection{Isotropy of Representation}
The autocorrelation of representation is defined as $\mathcal{C}_{\text{auto}} = \bm{H}^T\bm{H}/N \in \mathbb{R}^{d \times d}$ where $d$ is the representation vector's size. In contrast, isotropy concerns $\bm{H}\bm{H}^T \in \mathbb{R}^{N \times N}$ because it handles the uniformity in all orientations for the $N$ representation vectors in the $d$-dimensional vector space. Similar to the rank and disentanglement, we first provide a theoretical result. 

\begin{theorem}[Isotropy and VNE]
\label{thm:isotropy}
For a given representation matrix $\bm{H}\in \mathbb{R}^{N \times d}$,
suppose that $N \le d$ and $S(\mathcal{C}_{\text{auto}})$ is maximized. Then,
\vspace{-0.4cm}
\begin{equation}
\bm{H}\bm{H}^T = \bm{I}_N.
\end{equation}
\end{theorem}
Refer to Supplementary~\ref{sec:prop_proof} for the proof.
Theorem~\ref{thm:isotropy} states that if $S(\mathcal{C}_{\text{auto}})$ is maximized, representation vectors are uniformly distributed in all orientations and thus isotropic~\cite{arora2015latent}.
To perform an empirical analysis, we follow the studies of \cite{arora2015latent,mu2017all} and adopt the partition function $Z(\bm{c})=\sum_{i=1}^N\text{exp}(\bm{c}^T\bm{h}_i)$ defined for an arbitrary unit column vector $\bm{c}$. The partition function becomes constant when $\{\bm{h}_1, \cdots, \bm{h}_i, \cdots, \bm{h}_N\}$ are isotropically distributed. To be specific, the normalized partition function, $\frac{Z(\bm{c})}{\max_{||\bm{c}||=1}Z(\bm{c})}$, should become approximately 1 when the representation is isotropic (Lemma 2.1 in~\cite{arora2015latent}). We have analyzed the normalized partition function for meta-learning and supervised learning, and the obtained results are presented in Figure~\ref{figure:analysis_sim_reps}. In both cases, it can be observed that isotropy is strengthened by VNE$^+$ and weakened by VNE$^-$.
For domain generalization, the trend is the same, but the shift in the distribution turns out to be relatively limited (see Figure~\ref{figure:supple_iso} in Supplementary~\ref{sec:supplementary_results}). Based on the theoretical and empirical results, we can infer that the von Neumann entropy can be utilized as a proxy for controlling the representation's isotropy.

\begin{figure}[t!]
\centering
\subfloat[Meta-learning]{
\includegraphics[width=0.45\columnwidth]{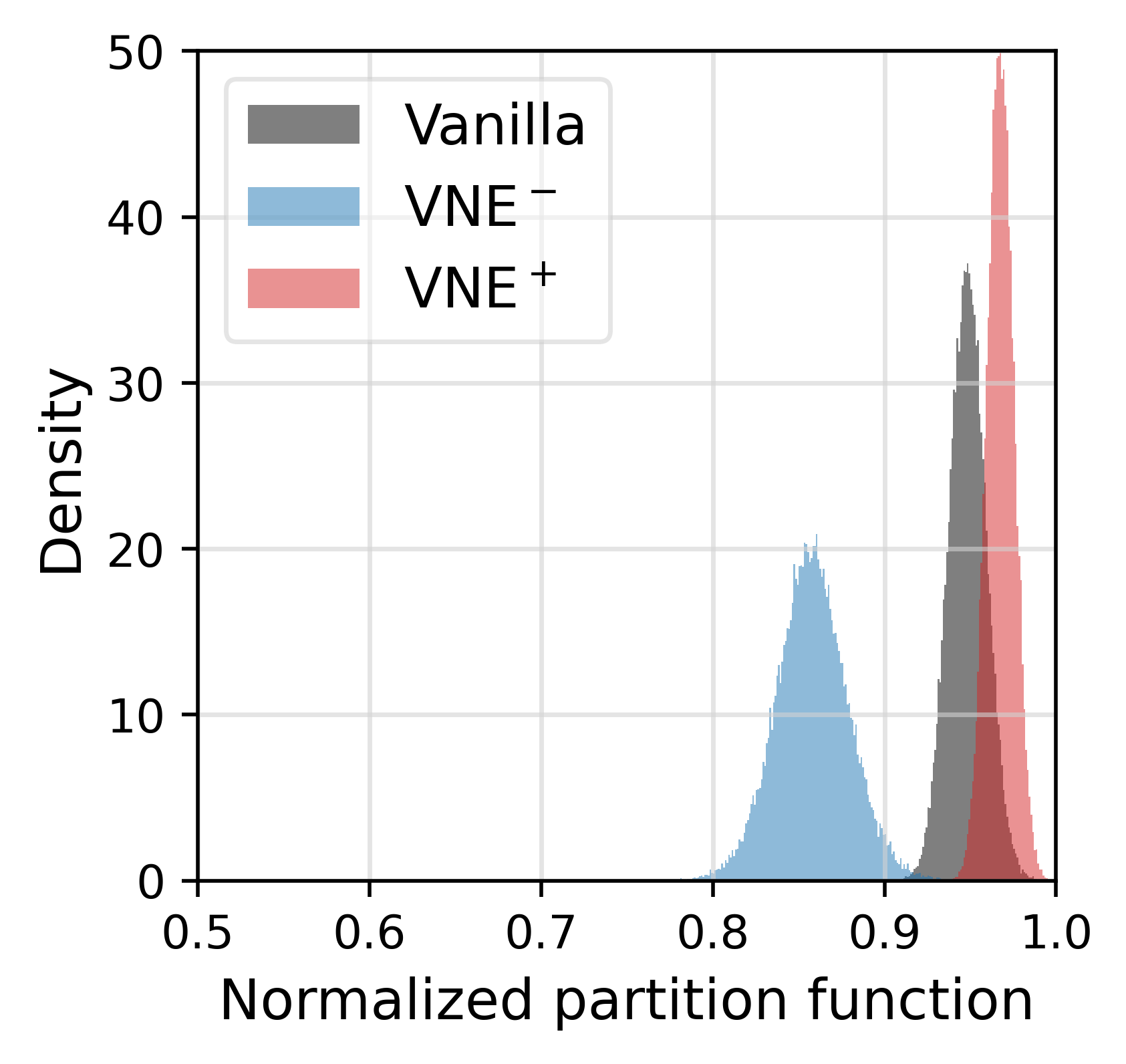}
}
\subfloat[Supervised (ImageNet-100)]{
\includegraphics[width=0.45\columnwidth]{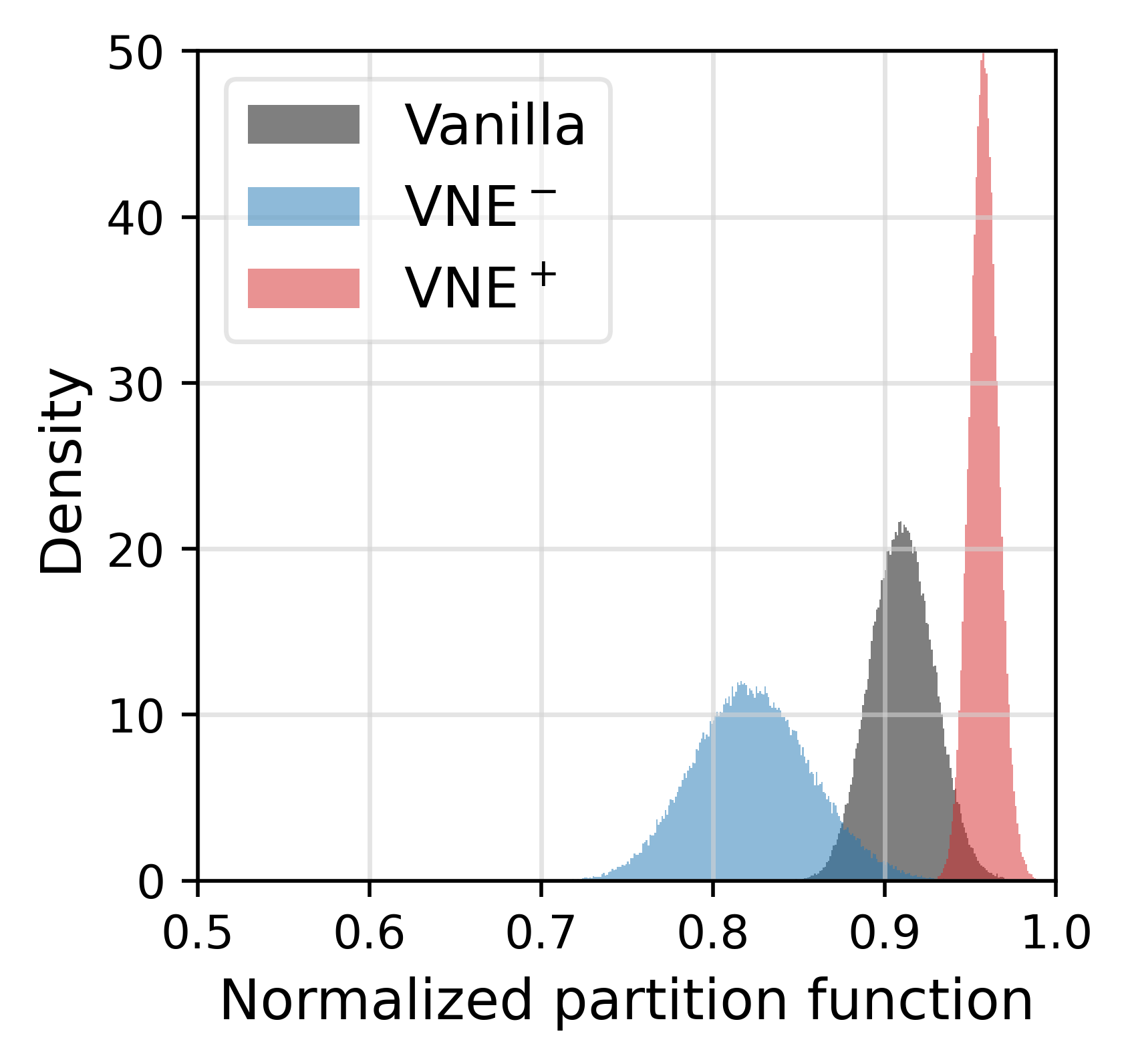}
}
\vspace{-0.2cm}
\caption{\label{figure:analysis_sim_reps}Isotropy of representation:
Distribution of the normalized partition function values.
Representation with its distribution closer to one is more isotropic.
}
\vspace{-0.45cm}
\end{figure}

\vspace{-0.1cm}
\section{Discussion}
\label{sec:discussion}
\vspace{-0.05cm}

\subsection{VNE and Dimension}
\label{sec:discussion_dg}
Although a large amount of information can be contained in a representation,
it is known that the \textit{usable information} is intimately linked to the predictive models that have computational constraints~\cite{xu2020theory,dubois2020learning}.
For instance, the representation decodability can be a critical factor when performing a linear evaluation~\cite{alain2016understanding}. From this perspective, decorrelation, disentanglement, whitening, and isotropy can be understood as improving decodability by encouraging a representation to use as many dimensions as possible with a full utilization of each dimension. Von Neumann entropy can be understood in the same way, except that its mathematical formulation is superior as explained in Section~\ref{sec:vonneumannentropy}. 

In this context, it looks logical that VNE$^+$ is beneficial in improving the performance of meta-learning, SSL, and GAN. However, for domain generalization, VNE$^+$ is harmful and VNE$^-$ is helpful. DG differs from the other tasks because the model needs to be ready for the same label-set but unseen target domains. Fine-tuning to the target domain is not allowed, either. In this case, the model needs to be trained to be solely dependent on the \textit{invariant features} and not on the \textit{spurious features}~\cite{arjovsky2019invariant, ahuja2021invariance, aubin2021linear, krueger2021out}. Because it is important to discard spurious features in DG, it makes sense that VNE$^-$ can be beneficial in reducing the number of dimensions and thus reducing the amount of usable information. However, if a very strong VNE$^-$ is applied, it can be harmful because even invariant features can be discarded.

\subsection{Von Neumann Entropy vs. Shannon Entropy}
Von Neumann entropy is defined over the representation autocorrelation $\mathcal{C}_{\text{auto}}$. For the representation $\bm{h}$ itself, Shannon entropy can be defined and it is relevant because it is also a metric of entropy. In fact, it can be proven that von Neumann entropy is a lower bound of Shannon entropy~\cite{nielsen2002quantum}. 

Owing to the connection, we have investigated if Shannon Entropy~(SE) can replace von Neumann entropy and achieve a better performance. Unlike VNE, however, regularizing Shannon metric is known to be difficult~\cite{kraskov2004estimating,gao2015efficient} and its implementation can be challenging. In our investigation, we have focused on the fact that Shannon entropy is equivalent to Shannon self-information (i.e., $H(\bm{h})= I(\bm{h};\bm{h})$~\cite{cover1999elements}) and that self-information can be evaluated using the latest variational mutual information estimators. In particular, we have chosen InfoNCE~\cite{oord2018representation,poole2019variational} as the mutual information estimator and regularized the Shannon entropy. An exemplary result for domain generalization is presented in Table~\ref{tab:shannon_entropy_dg} of Supplementary~\ref{sec:supplementary_results}. From the result, it can be observed that Shannon entropy can also improve the performance of ERM and SWAD. However, the overall improvement is smaller where the average improvements are 1.43\% and 0.79\% for VNE and SE, respectively. We have performed a similar comparison for SSL and reached the same conclusion. Although Shannon entropy is closely related to von Neumann entropy, the difficulty in manipulating Shannon entropy appears to make it less useful.  

\vspace{-0.05cm}
\section{Conclusion}
\label{sec:conclusion}
\vspace{-0.05cm}

In this study, we have proposed von Neumann entropy for manipulating the eigenvalue distribution of the representation's autocorrelation matrix $\mathcal{C}_{\text{auto}}$. We have shown why its mathematical formulation can be advantageous when compared to the conventional approach of Frobenius norm. Then, we have demonstrated von Neumann entropy's general applicability by empirically investigating four major learning tasks: DG, meta-learning, SSL, and GAN. Finally, we have established von Neumann entropy's theoretical connection with the conventional properties of rank, disentanglement, and isotropy. Overall, we conclude that von Neumann entropy is an effective and useful representation property for improving task performance.

\section*{Acknowledgements}
This work was supported by the following grants funded by the Korea government: NRF-2020R1A2C2007139, NRF-2022R1A6A1A03063039, and [NO.2021-0-01343, Artificial Intelligence Graduate School Program (Seoul National University)].

{\small
\bibliographystyle{ieee_fullname}
\bibliography{reference}
}

\clearpage
\appendix

\twocolumn[{
 \centering
 \Large Supplementary materials for the paper \\ ``VNE: An Effective Method for Improving Deep Representation \\ by Manipulating Eigenvalue Distribution''\\[1em]
}]

\section{A Brief Introduction to Quantum Theory}
\label{sec:quantum_prelim}
A classic bit can be either 0 or 1. In quantum theory~\cite{nielsen2002quantum, wilde2013quantum}, a \textit{qubit} is a quantum extension of the classic bit, and it can be in state $\ket{0}$, state $\ket{1}$, or any linear combination (superposition state) of the two as $\ket{\psi} = a\ket{0} + b\ket{1}$, where $|a|^2+|b|^2=1$.
\vspace{-0.2cm}
\paragraph{Dirac notation and basic concepts:}
Dirac notation is used in quantum theory~\cite{dirac1939new}. For a state $\ket{\psi}$, $\psi$ should be understood as the name or label of the state. Because linear algebra provides the mathematical foundation of quantum theory, vector notation is adopted. For instance, in the simple example of  $\ket{\psi}= a\ket{0} + b\ket{1}$, $\ket{\psi}$ can be expressed as $\ket{\psi} = [a, b]^T$ where the interpretation should be state $\ket{\psi}$ can be 0 with probability $|a|^2$ and 1 with probability $|b|^2$ (therefore $|a|^2+|b|^2=1$). Here, the \textit{ket} vector $\ket{\psi}$ is the Dirac notation for a column vector in a Hilbert space $\mathcal{H}$. To represent a row vector, the \textit{bra} vector $\bra{\psi}$ is used, as in $\bra{\psi}=[a, b]$. An inner product or \textit{braket} is represented as $\braket{\psi}{\phi}$ and an outer product or \textit{ketbra} is represented as $\ketbra{\psi}{\phi}$.

A \textit{composite quantum state} of $n$ qubits can be represented as a vector of size $2^n$ (e.g., a single-qubit state is represented as a vector of size two). For example, a quantum state of two separable single-qubit states can be represented as
\vspace{-0.1cm}
\begin{equation}
\begin{split}
\ket{\psi}\otimes\ket{\phi} &= \ket{\psi}\ket{\phi} = \ket{\psi\phi}\\
&= [a,b]^T\otimes[c,d]^T = [ac,ad,bc,bd]^T
\end{split}
\vspace{-0.1cm}
\end{equation}
in which $\abs{ac}^2,\abs{ad}^2,\abs{bc}^2,$ and $\abs{bd}^2$ represent the probability of $\ket{\psi\phi}$ being $\ket{00}, \ket{01}, \ket{10}$, and $\ket{11}$, respectively.
In $d$-dimensional quantum system, a quantum state is on the unit hypersphere in a Hilbert space $\mathcal{H}$.

A state can be either \textit{pure} or \textit{mixed}. In the simple example, $\ket{0}=[1, 0]^T$ and $\ket{1}=[0, 1]^T$ form the \textit{computational basis states}, and they are pure states. Any superposition of the two, $\ket{\psi}=a\ket{0} + b\ket{1}$, is also a pure state because it corresponds to a single vector with a probabilistic distribution over the basis states. By contrast, a mixed state is a probabilistic mixture of a set of pure states. Note that a pure state already has a probabilistic interpretation over the basis states and a mixed state has an additional level of probabilistic interpretation over a set of such pure states. 
In this case, we are considering a state that is not completely known but is an ensemble of pure states $\{\ket{\psi_{i}}\}$ with respective probabilities $\{p_i\}$.
The full information of a mixed state cannot be represented as a vector, and the notion of the density operator (also called density matrix) is required.
\begin{definition}[Density operator~\cite{nielsen2002quantum}]
\label{def:density}
A density operator is defined as below.
\vspace{-0.2cm}
\begin{equation}
\rho\triangleq \sum_{i}p_i\ketbra{\psi_{i}}{\psi_{i}}.
\vspace{-0.2cm}
\end{equation}
\end{definition}
Density operator $\rho$ satisfies $\rho \ge 0$ and $tr(\rho)=1$. In addition, $\rho = \rho^2$ and $rank(\rho) = 1$ are satisfied for pure states and $tr(\rho^2)<1$ is satisfied for mixed states.
The density operator provides a convenient way to describe the uncertainty or probability distribution of a quantum system. According to Gleason's theorem~\cite{gleason1957measures}, the probability of a state $\ket{\psi_i}$ in the system with $\rho$ is given by $tr(\rho\ketbra{\psi_i}{\psi_i})$.

While quantum theory encompasses a broad scope of subjects, quantum information theory or quantum Shannon theory is a sub-field that focuses on the quantum equivalent of Shannon information theory~\cite{wilde2013quantum}.
Among the extensive results, we utilize the basic concepts of \textit{von Neumann entropy} (also called quantum entropy).
While Shannon entropy is calculated for a classical probability distribution, von Neumann entropy is calculated for a density operator $\rho$~\cite{nielsen2002quantum}, a positive semi-definite hermitian matrix in a Hilbert space $\mathcal{H}$ with the trace value of one.
Similar to Shannon information theory, it measures the uncertainty associated with a quantum system.
\begin{definition}[von Neumann entropy~\cite{nielsen2002quantum}]
\label{def:entropy}
The von Neumann entropy (quantum entropy) of a quantum state with density operator $\rho$ is defined as
\vspace{-0.2cm}
\begin{equation}
S(\rho)\triangleq-tr(\rho\log{\rho})=-\sum_j \lambda_j\log{\lambda_j},
\vspace{-0.2cm}
\end{equation}
where \{$\lambda_j$\} are the eigenvalues of $\rho$.
\end{definition}

\section{Proofs of Theorems}
\label{sec:prop_proof}
\setcounter{theorem}{0}
\setcounter{definition}{0}
\setcounter{assumption}{0}

\begin{lemma}
\label{lem:entropy}
For given $p_{i}\ge0$ and $\sum_{i=1}^n p_{i}=1$, the entropy function $H(p_{1},...,p_{n})= -\sum_{i=1}^n p_{i} \log{p_{i}}$ is strictly concave and is upper-bounded by $\log{n}$ as follows,
\vspace{-0.2cm}
\begin{equation}
\label{eq:property_entropy}
\log{n} = H(1/n,...,1/n) \ge H(p_{1},...,p_{n}) \ge 0.
\vspace{-0.2cm}
\end{equation}
\end{lemma}
\begin{proof}
Refer to Section D.1 in~\cite{marshall1979inequalities}.
\end{proof}

\begin{lemma}
\label{lem:kld_gaussian}
The KL Divergence for two zero-mean $d$-dimensional multivariate Gaussian distributions can be derived as follows,
\vspace{-0.2cm}
\begin{equation}
\begin{split}
&D_{\mathrm{KL}}( \mathcal{N} (0, \bm{\Sigma}_1) \Vert \mathcal{N} (0, \bm{\Sigma}_2) )\\
&= \frac{1}{2}\left[\text{tr}(\bm{\Sigma}_2^{-1}\bm{\Sigma}_1)-d +\log\frac{|\bm{\Sigma}_2|}{|\bm{\Sigma}_1|} \right].
\vspace{-0.2cm}
\end{split}
\end{equation}
\end{lemma}
\begin{proof}
Refer to Section 9 in~\cite{duchi2007derivations}.
\end{proof}

\begin{theorem}[Rank and VNE]
For a given representation autocorrelation $\mathcal{C}_{\text{auto}} = \bm{H}^T\bm{H}/N \in \mathbb{R}^{d \times d}$ of rank $k$ $(\le d)$,
\vspace{-0.1cm}
\begin{equation}
\text{log}(\text{rank}(\mathcal{C}_{\text{auto}})) \ge S(\mathcal{C}_{\text{auto}}),
\end{equation}
where equality holds iff the eigenvalues of $\mathcal{C}_{\text{auto}}$ are uniformly distributed with $\forall_{j=1}^k \lambda_{j}=1/k$ and $\forall_{j=k+1}^d \lambda_{j}=0$.
\end{theorem}
\begin{proof}
\begin{align}
\text{log}(\text{rank}(\mathcal{C}_{\text{auto}})) &= \text{log}(k)\\[-2pt]
\label{eq:inequlity_rank}
&\ge H(\lambda_{1},...,\lambda_{k}) \text{ (by Lemma~\ref{lem:entropy})}\\[-2pt]
&=-\sum_{j=1}^k \lambda_{j} \log{\lambda_{j}}\\[-2pt]
\label{eq:equlity_rank}
&=-\sum_{j=1}^d \lambda_{j} \log{\lambda_{j}}\\[-2pt]
&=S(\mathcal{C}_{\text{auto}}).
\end{align}
By Lemma~\ref{lem:entropy}, the inequality~\eqref{eq:inequlity_rank} holds with equality if and only if $\forall_{j=1}^k \lambda_{j}=1/k$.
The Eq.~\eqref{eq:equlity_rank} follows from the convention $0\log{0} = 0$~\cite{cover1999elements}. 
\end{proof}

\begin{assumption}
\label{assumption:zero_mean_gaussian_supple}
We assume that representation $\bm{h}$ follows zero-mean multivariate Gaussian distribution.
In addition, we assume that the components of $\bm{h}$ (denoted as $\bm{h}^{(i)}$) have homogeneous variance of $\frac{1}{d}$, i.e., $\forall_{i=1}^d\bm{h}^{(i)} \sim \mathcal{N} (0, \frac{1}{d})$.
\end{assumption}

\begin{theorem}[Disentanglement and VNE]
Under the Assumption~\ref{assumption:zero_mean_gaussian_supple}, $\bm{h}$ is disentangled if $S(\mathcal{C}_{\text{auto}})$ is maximized.
\end{theorem}

\begin{proof}
By Assumption~\ref{assumption:zero_mean_gaussian_supple},
$\bm{h} \sim \mathcal{N} (0, \bm{\Sigma}_1)$ for $\bm{\Sigma}_1\in \mathbb{R}^{d \times d}$
where diagonal entries in $\bm{\Sigma}_1$ are equal to $1/d$.

In addition, we define new random variable $\bm{h}' \sim \mathcal{N} (0, \bm{\Sigma}_2)$ for $\bm{\Sigma}_2=\frac{1}{d}\cdot I_d$.

Then, because $\bm{h}^{(i)}\sim\mathcal{N} (0, \frac{1}{d})$ and $\bm{h}'^{(i)}\sim\mathcal{N} (0, \frac{1}{d})$
and the components of $\bm{h}'$ are independent,
\vspace{-0.1cm}
\begin{equation}
\label{eq:disentanglement_prop_1}
\prod_{i=1}^d p(\bm{h}^{(i)}) = \prod_{i=1}^d p(\bm{h}'^{(i)}) = p(\bm{h}').
\vspace{-0.1cm}
\end{equation}
By Lemma~\ref{lem:entropy}, $S(\mathcal{C}_{\text{auto}})$ is maximized if and only if
\vspace{-0.1cm}
\begin{equation}
\label{eq:disentanglement_prop_2}
\forall_{j=1}^d\lambda_{j}=\frac{1}{d},
\vspace{-0.1cm}
\end{equation}
where $\lambda_{j}$ are eigenvalues of $\bm{\Sigma}_1(=\mathbb{E}[\bm{h}\bm{h}^T]=\mathcal{C}_{\text{auto}})$.

Starting from Definition of total correlation $TC(\bm{h})$ in~\cite{achille2018emergence}, we have
\vspace{-0.2cm}
\begin{align}
2\cdot TC(\bm{h})&=2\cdot D_{\mathrm{KL}}( p(\bm{h}) \Vert \prod_{i=1}^d p(\bm{h}^{(i)}) )\\[-2pt]
\label{eq:disentanglement_tc_1}
&=2\cdot D_{\mathrm{KL}}( p(\bm{h}) \Vert p(\bm{h}') )\\[-2pt]
\label{eq:disentanglement_tc_2}
&=\text{tr}(\bm{\Sigma}_2^{-1}\bm{\Sigma}_1) -d +\log{\frac{|\bm{\Sigma}_2|}{|\bm{\Sigma}_1|}}\\[-2pt]
\label{eq:disentanglement_tc_3}
&=d -d +\log{\frac{(1/d)^d}{(1/d)^d}}=0,
\end{align}
where Eq.~\eqref{eq:disentanglement_tc_1} follows from Eq.~\eqref{eq:disentanglement_prop_1}, Eq.~\eqref{eq:disentanglement_tc_2} follows from Lemma~\ref{lem:kld_gaussian}, and Eq.~\eqref{eq:disentanglement_tc_3} follows from Eq.~\eqref{eq:disentanglement_prop_2}.

If $TC(\bm{h})=0$, the components of $\bm{h}$ are independent, therefore $\bm{h}$ is disentangled~\cite{achille2018emergence}.
\end{proof}

\begin{theorem}[Isotropy and VNE]
For a given representation matrix $\bm{H}\in \mathbb{R}^{N \times d}$,
suppose that $N \le d$ and $S(\mathcal{C}_{\text{auto}})$ is maximized. Then,
\vspace{-0.1cm}
\begin{equation}
\bm{H}\bm{H}^T = \bm{I}_N.
\vspace{-0.1cm}
\end{equation}
\end{theorem}

\begin{proof}
We consider singular value decomposition of $\bm{H} (= \bm{U}\bm{\Sigma}\bm{V}^T)$ for $\bm{U}\in \mathbb{R}^{N \times N}$, $\bm{\Sigma}\in \mathbb{R}^{N \times d}$, and $\bm{V}\in \mathbb{R}^{d \times d}$.
If $N \le d$ and $S(\mathcal{C}_{\text{auto}})$ is maximized, by Lemma~\ref{lem:entropy}, eigenvalues of $\mathcal{C}_{\text{auto}}(=\bm{H}^T\bm{H}/N=\bm{V}\bm{\Sigma}^T\bm{\Sigma}\bm{V}^T/N)$ are supposed to be equal to $1/N$ for the first $N$ eigenvalues and zero for the others. Therefore $\bm{\Sigma}\bm{\Sigma}^T=\bm{I}_N$ and we have
\vspace{-0.1cm}
\begin{equation}
\bm{H}\bm{H}^T=\bm{U}\bm{\Sigma}\bm{\Sigma}^T\bm{U}^T = \bm{I}_N.
\vspace{-0.1cm}
\end{equation}
\end{proof}

\section{Main Algorithm}
\label{sec:main_algo}

\begin{figure}[h!]
\centering
\begin{minipage}{\columnwidth}
\begin{minted}[frame=single,fontsize=\scriptsize]{python}
# N   : batch size
# d   : embedding dimension
# H   : embeddings, Tensor, shape=[N, d]

def get_vne(H):
    Z = torch.nn.functional.normalize(H, dim=1)
    rho = torch.matmul(Z.T, Z) / Z.shape[0]
    eig_val = torch.linalg.eigh(rho)[0][-Z.shape[0]:]
    return - (eig_val * torch.log(eig_val)).nansum()

# the following is equivalent and faster when N < d
def get_vne(H):
    Z = torch.nn.functional.normalize(H, dim=1)
    sing_val = torch.svd(Z / np.sqrt(Z.shape[0]))[1]
    eig_val = sing_val ** 2
    return - (eig_val * torch.log(eig_val)).nansum()
\end{minted}
\caption{\label{figure:vne_code}PyTorch implementation of VNE.}
\end{minipage}
\end{figure}

\section{Computational Overhead}

We train I-VNE$^+$ using 2$\times$RTX 3090 GPUs, ImageNet-1K, and various batch sizes and models. In Table~\ref{tab:computational_overhead}, the average computational overhead is 2.68\%.

\label{sec:computational_overhead}
\begin{table}[h!]
\centering
\resizebox{\columnwidth}{!}{
\begin{tabular}{@{}llcccccc@{}}
\toprule
Model                 &        & \multicolumn{3}{c}{ResNet-18} & \multicolumn{3}{c}{ResNet-50} \\ \cmidrule(){1-2} \cmidrule(l){3-5} \cmidrule(l){6-8}
Batch Size            &        & 256      & 128      & 64      & 256      & 128      & 64      \\ \midrule
Average training time & On VNE & 0.051    & 0.024    & 0.011   & 0.120    & 0.073    & 0.031   \\
per iteration (sec.)  & Total  & 2.318    & 1.288    & 0.845   & 2.745    & 2.101    & 1.127   \\ \midrule
Overhead              &        & 2.21\%   & 1.89\%   & 1.36\%  & 4.37\%   & 3.48\%   & 2.75\%  \\ \bottomrule
\end{tabular}
}
\caption{\label{tab:computational_overhead}Computational overhead of VNE.}
\end{table}

\section{Experimental Details for I-VNE$^+$}
\label{sec:implementation_delail}
The PyTorch implementation codes will be made available online.
Our implementations follow the standard training protocols of SSL in~\cite{grill2020bootstrap,zbontar2021barlow} and the standard evaluation protocols of SSL in~\cite{misra2020self,goyal2019scaling,grill2020bootstrap,zbontar2021barlow}.
A few important hyperparameters are described as follows.

\textbf{Backbone and Projector:}
For all datasets, we use ResNet-50~\cite{he2016deep} as the default backbone.
For CIFAR-10, we use 2-layer MLP projector with hidden dimension of 2048 and output dimension of 128.
For ImageNet-100, we use 3-layer MLP projector with hidden dimension of 2048 and output dimension of 256.
For ImageNet-1K, we use the same projector as in the  ImageNet-100 case, except that the output dimension is 512.

\textbf{Optimization:}
We use SGD optimizer with momentum of 0.9. The learning rate~(LR) is linearly scaled with batch size (LR = base learning rate $\times$ batch size / 256), and it is scheduled by the cosine learning rate decay with 10-epoch warm-up~\cite{loshchilov2016sgdr}.
For CIFAR-10 and ImageNet-100, we use base learning rate of 0.4, batch size of 64, and weight decay of 1e-4.
For ImageNet-1K, we use base learning rate of 0.2, batch size of 512, and weight decay of 1e-5.

\textbf{Augmentation:}
For CIFAR-10 and ImageNet-100, we adopt multi-view setting in~\cite{caron2020unsupervised} and generate 6 views using the same augmentations in~\cite{chen2020simple} (for CIFAR-10) and in~\cite{caron2020unsupervised} (for ImageNet-100).
For ImageNet-1K, we generate the default 2 views using the same augmentation as in~\cite{grill2020bootstrap}.
Note that we use 2-view setting for ImageNet-1K because of the computational limitation.

\section{Supplementary Results}
\label{sec:supplementary_results}
\begin{table}[h!]
\centering
\resizebox{0.5\columnwidth}{!}{
\begin{tabular}{@{}lcc@{}}
\toprule
Method                                                 & Top-1         & Top-5         \\ \midrule
Supervised \cite{chen2020simple}      & 76.5          & 93.7          \\ \midrule
SimCLR \cite{chen2020simple}          & 69.3          & 89.0          \\
MoCo v2 \cite{chen2020improved}       & 71.1          & 90.1          \\
InfoMin Aug. \cite{tian2020makes}     & 73.0          & 91.1          \\
BYOL \cite{grill2020bootstrap}        & 74.3          & \textbf{91.6} \\
SwAV \cite{caron2020unsupervised}     & \textbf{75.3} &               \\ 
Shuffled-DBN \cite{hua2021feature}    & 65.2          &               \\
Barlow Twins \cite{zbontar2021barlow} & 73.2          & 91.0          \\
VICReg \cite{bardes2021vicreg}        & 73.2          & 91.1          \\ \midrule
I-VNE$^+$ (ours)                                             & 72.1          & 91.0          \\ \bottomrule
\end{tabular}
}
\caption{\label{tab:ssl_linear_appendix}SSL: Linear evaluation performance in ImageNet-1K for various representation learning methods. They are all based on ResNet-50 encoders pre-trained with various datasets. Linear classifier on top of the frozen pre-trained model is trained with labels. State-of-the-art methods are included and the best results are indicated in bold.
}
\end{table}

\phantom{0}\\
\begin{figure}[ht!]
\centering
\begin{minipage}{0.48\columnwidth}
\includegraphics[width=\columnwidth]{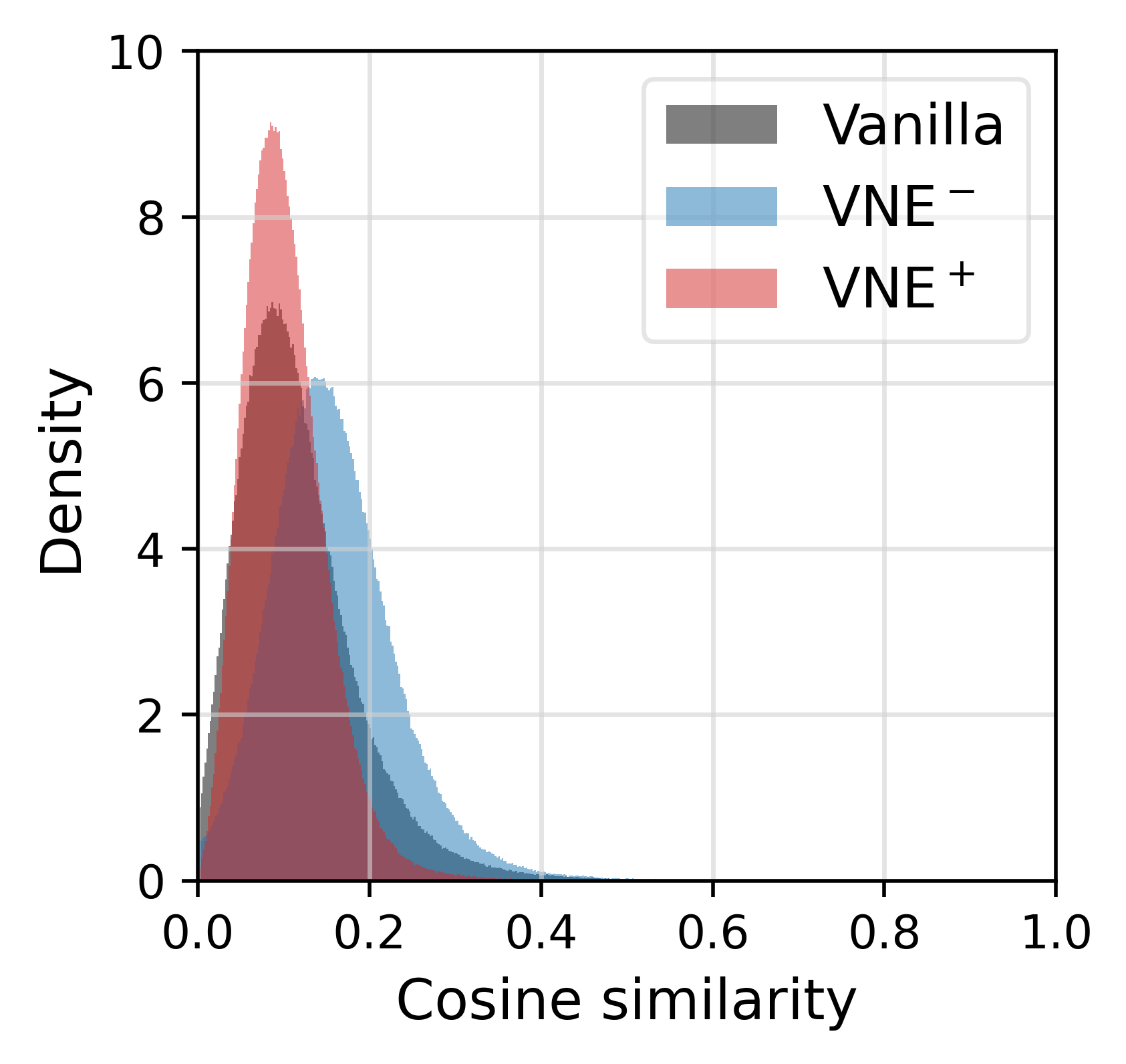}
\caption{\label{figure:supple_dis}Meta-learning: Disentanglement of representation.
}
\end{minipage}
\hfill
\begin{minipage}{0.48\columnwidth}
\includegraphics[width=\columnwidth]{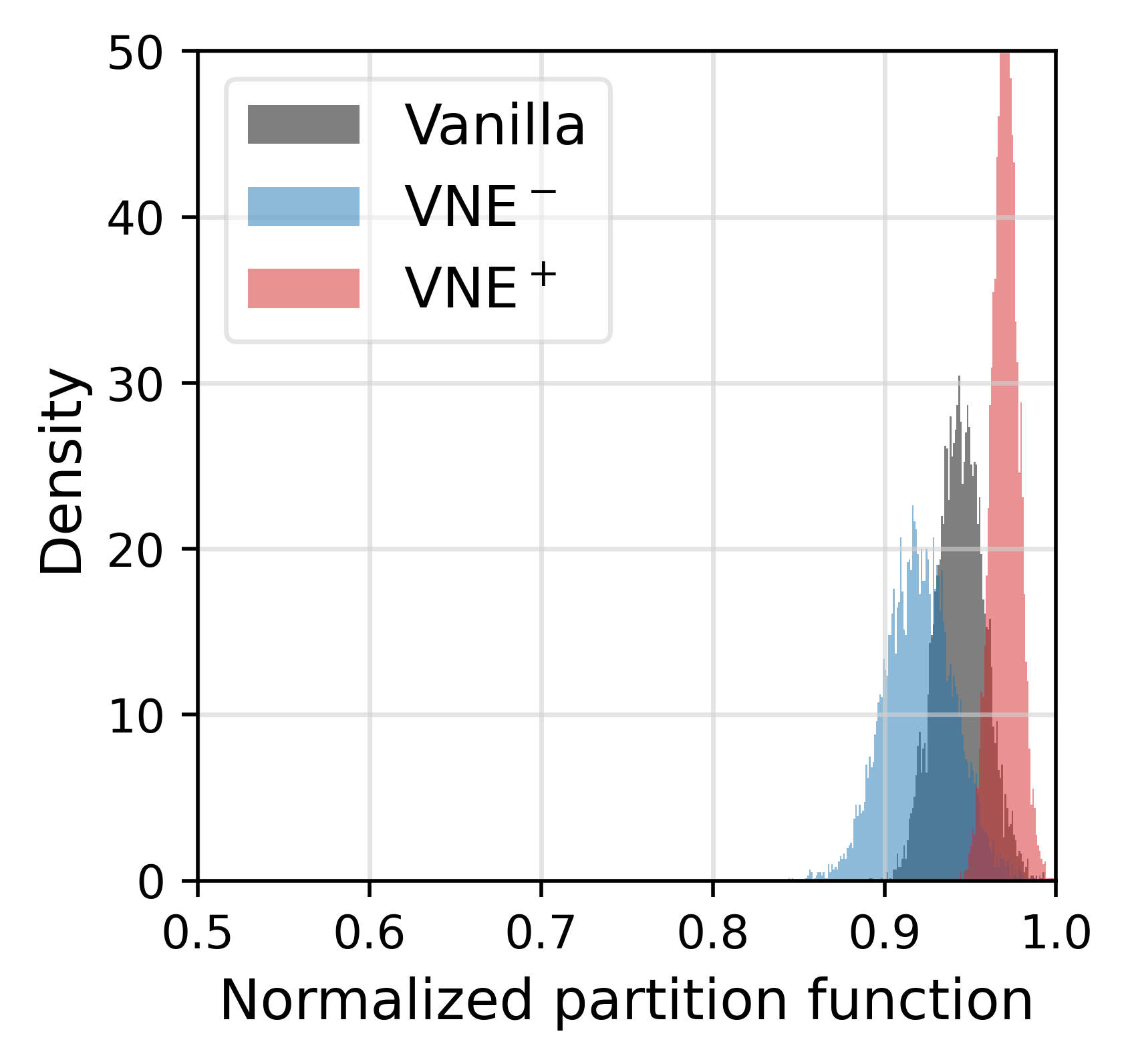}
\caption{\label{figure:supple_iso}Domain generalization: Isotropy of representation.
}
\end{minipage}
\vspace{-1.0cm}
\end{figure}

\phantom{0}\\
\begin{table}[h!]
\centering
\resizebox{\columnwidth}{!}{
\begin{tabular}{@{}llcccccccc@{}}
\toprule
Algorithm & Method  & \multicolumn{2}{c}{PACS} & \multicolumn{2}{c}{VLSC} & \multicolumn{2}{c}{OfficeHome} & \multicolumn{2}{c}{TerraIncognita} \\ \cmidrule(l){3-4} \cmidrule(l){5-6} \cmidrule(l){7-8} \cmidrule(l){9-10} 
          &         & Avg.    & Diff.          & Avg.    & Diff.          & Avg.       & Diff.             & Avg.         & Diff.               \\ \midrule
ERM       & Vanilla & 85.2    &                & 76.7    &                & 64.9       &                   & 45.4         &                     \\ \cmidrule(l){2-10} 
          & VNE$^-$ & 86.9    & \textbf{1.7}   & 78.1    & \textbf{1.4}   & 65.9       & \textbf{1.0}      & 50.6         & \textbf{5.2}        \\
          & SE$^-$ & 85.0    & -0.2           & 76.5    & -0.2           & 65.3       & 0.4               & 50.4         & 5.0                 \\ \midrule
SWAD      & Vanilla & 88.2    &                & 79.4    &                & 70.2       &                   & 50.9         &                     \\ \cmidrule(l){2-10} 
          & VNE$^-$ & 88.3    & 0.1            & 79.7    & \textbf{0.3}   & 71.1       & \textbf{0.9}      & 51.7         & \textbf{0.8}        \\
          & SE$^-$ & 88.4    & \textbf{0.2}   & 79.6    & 0.1            & 71.0       & 0.8               & 51.2         & 0.2                 \\ \bottomrule
\end{tabular}
}
\caption{\label{tab:shannon_entropy_dg} Von Neumann entropy vs. Shannon entropy: The results of domain generalization with ERM and SWAD algorithms are shown. For regularizing Shannon entropy, we have used the InfoNCE estimation of self-information, $I_{\text{NCE}}(\bm{h};\bm{h})$. 
}
\end{table}

\end{document}